\documentclass{article}
\usepackage{max-style}

\usepackage[top=1in, bottom=1in, left=1in, right=1in]{geometry}

\usepackage{tikz}
\usepackage{pgfplots}
\usetikzlibrary{arrows.meta}
\pgfplotsset{compat=1.17}

\PassOptionsToPackage{numbers, compress}{natbib}

\usepackage{lipsum}

\usepackage[utf8]{inputenc}
\usepackage[T1]{fontenc}

\usepackage{longtable,tabularx}
\usepackage{microtype}
\usepackage{multicol}
\usepackage{multirow}
\usepackage{caption}
\usepackage{savesym}
\usepackage{algorithm}
\usepackage{algorithmic}
\usepackage{subfigure}

\usepackage{url}
\usepackage{mathrsfs}
\usepackage{nicefrac,bbm}
\usepackage{hyperref}
\hypersetup{
  colorlinks=true,
  linkcolor=green!30!black,
  linktoc=all,
  citecolor=blue,
  bookmarksnumbered=true,
  pdfproducer={},
  pdfinfo={CreationDate={2025},ModDate={2025}}, 
  pdfauthor={},
  pdfkeywords={},
  pdfsubject={},
  pdfcreator={},
  pdflang={}
}
\usepackage{cleveref}
\usepackage{mathtools,pifont,enumitem}
\usepackage[framemethod=TikZ]{mdframed}
\usepackage{color,fancybox,graphicx,subfigure}

\savesymbol{st}
\usepackage{soul}
\restoresymbol{SOUL}{st}
\usepackage[normalem]{ulem}

\usepackage{xurl}
\usepackage{scalerel}
\usetikzlibrary{math}
\usepackage{bbding}
\usepackage{array}
\usepackage{dsfont}

\usepackage{empheq}
\usepackage{stackengine}

\newtheorem{assumption}{Assumption}
\newtheorem{theorem}{Theorem}[section]
\newtheorem{lemma}[theorem]{Lemma}
\newtheorem{definition}[theorem]{Definition}

\newtheorem{corollary}[theorem]{Corollary}

\newtheorem{proposition}[theorem]{Proposition}

\newtheorem{remark}[theorem]{Remark}

\crefname{section}{Section}{Sections}
\crefname{theorem}{Theorem}{Theorems}
\crefname{assumption}{Assumption}{Assumptions}
\crefname{lemma}{Lemma}{Lemmas}
\crefname{definition}{Definition}{Definitions}
\crefname{conjecture}{Conjecture}{Conjectures}
\crefname{corollary}{Corollary}{Corollaries}
\crefname{construction}{Construction}{Constructions}
\crefname{claim}{Proposition}{Propositions}
\crefname{observation}{Observation}{Observations}
\crefname{proposition}{Proposition}{Propositions}
\crefname{fact}{Fact}{Facts}
\crefname{question}{Question}{Questions}
\crefname{problem}{Problem}{Problems}
\crefname{remark}{Remark}{Remarks}
\crefname{example}{Example}{Examples}
\crefname{equation}{Equation}{Equations}
\crefname{appendix}{Appendix}{Appendices}
\crefname{algorithm}{Algorithm}{Algorithms}
\crefname{model}{Model}{Models}
\crefname{figure}{Figure}{Figures}

\AfterEndEnvironment{definition}{\noindent\ignorespaces}
\AfterEndEnvironment{assumption}{\noindent\ignorespaces}
\AfterEndEnvironment{lemma}{\noindent\ignorespaces}
\AfterEndEnvironment{theorem}{\noindent\ignorespaces}
\AfterEndEnvironment{proposition}{\noindent\ignorespaces}
\AfterEndEnvironment{fact}{\noindent\ignorespaces}
\AfterEndEnvironment{question}{\noindent\ignorespaces}

\hypersetup{
  colorlinks=true,
  linkcolor=green!30!black,
  linktoc=all,
  citecolor=blue,
  bookmarksnumbered=true,
  pdfproducer={},
  pdfinfo={CreationDate={2025},ModDate={2025}}, 
  pdfauthor={},
  pdfkeywords={},
  pdfsubject={},
  pdfcreator={},
  pdflang={}
}

\newcommand{\mbf}[1]{\mathbf{#1}}

\DeclareMathOperator*{\E}{\mathbb{E}}
\DeclareMathOperator{\KL}{D_{KL}}
\DeclareMathOperator{\supp}{supp}

\savesymbol{Bbbk}

\newcommand{\R}{\mathbb{R}}

\newcommand{\cD}{\mathcal{D}}

\newcommand{\cX}{\mathcal{X}}
\newcommand{\cY}{\mathcal{Y}}

\renewcommand{\epsilon}{\varepsilon}

\renewcommand{\E}{\operatornamewithlimits{\mathbb{E}}}

\AfterEndEnvironment{definition}{\noindent\ignorespaces}
\AfterEndEnvironment{assumption}{\noindent\ignorespaces}
\AfterEndEnvironment{lemma}{\noindent\ignorespaces}
\AfterEndEnvironment{theorem}{\noindent\ignorespaces}
\AfterEndEnvironment{proposition}{\noindent\ignorespaces}
\AfterEndEnvironment{fact}{\noindent\ignorespaces}
\AfterEndEnvironment{question}{\noindent\ignorespaces}

\usepackage{color-edits}

\usepackage{xspace}

\addauthor[Aleksandra]{ak}{orange}
\addauthor[Hayoung]{hy}{olive}

\addauthor{jc}{purple}

\usepackage{xfrac}

\usepackage[most]{tcolorbox}
\usepackage{xcolor}

\newtcolorbox{blackbannerbox}[1][]{
  colback=gray!5,        
  colframe=black,        
  coltitle=white,        
  fonttitle=\bfseries,   
  colbacktitle=black,    
  enhanced,
  #1
}

\title{The Geometry of Alignment Collapse: When Fine-Tuning Breaks Safety}

\author{
    Max Springer\textsuperscript{1,*}, 
    Chung Peng Lee\textsuperscript{1}, 
    Blossom Metevier\textsuperscript{1}, 
    Jane Castleman\textsuperscript{1}, \\
    Bohdan Turbal\textsuperscript{1},
    Hayoung Jung\textsuperscript{1},
    Zeyu Shen\textsuperscript{1},
    Aleksandra Korolova\textsuperscript{1}
}

\affiliations{
    \textsuperscript{1}Princeton University \\
    \textsuperscript{*}Corresponding Author: \texttt{maxspringer@princeton.edu}
}

\paperabstract{

Fine-tuning aligned language models on entirely benign tasks (e.g. math tutoring) unpredictably degrades safety guardrails, even when training data contains no harmful content and developers have no adversarial intent. 
We show that the prevailing explanation, that fine-tuning updates should be orthogonal to safety-critical directions in high-dimensional parameter space, offers false reassurance: we show that this orthogonality is structurally unstable and collapses under the very dynamics of gradient descent. 
We then resolve this through a novel geometric analysis, proving that alignment concentrates in low-dimensional subspaces with sharp curvature, creating a brittle structure that first-order methods cannot detect or defend. 
While initial fine-tuning updates may indeed avoid these subspaces, the curvature of the fine-tuning loss generates second-order acceleration that systematically steers trajectories into alignment-sensitive regions—an effect invisible to all existing defenses. 
We formalize this mechanism through the \textbf{\underline{A}lignment \underline{I}nstability \underline{C}ondition (AIC)}, three geometric properties that, when jointly satisfied, lead to safety degradation. 
Our main result establishes a quartic scaling law: alignment loss grows with the fourth power of training time, governed by the sharpness of alignment geometry and the strength of curvature coupling between the fine-tuning task and safety-critical parameters.
These results expose a structural blind spot in the current safety paradigm. The dominant approaches to safe fine-tuning of null-space projections, gradient filtering, and first-order constraints address only the initial snapshot of a fundamentally dynamic problem. 
Alignment fragility is not a bug to be patched; it is an intrinsic geometric property of gradient descent on curved manifolds. Our results motivate the development of curvature-aware methods, and we hope will further enable a shift in alignment safety analysis from reactive red-teaming to predictive diagnostics for open-weight model deployment.}

\date{\today}

\begin{document}

\maketitle




\section{Introduction} \label{sec:intro}

Fine-tuning a language model on a dataset of math problems breaks its safety guardrails~\cite{cobbe2021training}.
Training on creative writing examples erodes medical advice safeguards~\cite{alkaeed2025open, yang2025adversarial}.
Adapting for code generation compromises refusal mechanisms~\cite{betley2025emergent}.
These are not hypothetical risks\textemdash{}they are documented, reproducible phenomena that occur even when fine-tuning datasets contain no harmful content and developers have no adversarial intent
~\cite{betley2025emergent,guan2025benign,he2024what,qi2023fine}.
The central puzzle of alignment preservation is thus:
\begin{center}
    \emph{Why does adapting a model for one capability seemingly unpredictably degrade unrelated safety properties?}
\end{center}
The question is urgent because fine-tuning has become the dominant paradigm for deploying language models.
In domains where technical expertise, compute resources, and specialized data are limited, fine-tuning open weight models represents the primary path to leverage AI capabilities for practitioners.
%
The ecosystem of open-weight models has expanded rapidly~\cite{bhandari2025forecasting,burton2024large,choksi2025brief,kapoor2024societal,osborne2024ai}, enabling innovation but also creating vulnerability.
Current defenses are largely reactive, targeting specific attack patterns~\cite{dong2025safeguarding} or imposing first-order gradient constraints~\cite{das2025alignguard,niu2025mitigating,zhang2025guardrail} that easily break even without adversarial intent~\cite{hsiung2025your,yang2025alleviating}.
Without understanding the underlying mechanisms of alignment degradation, we cannot build robust safeguards.


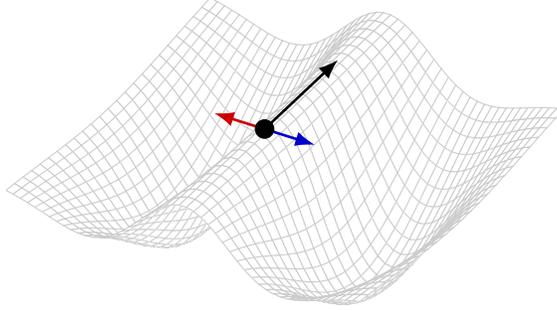
\begin{figure}[t]
    \centering
    \begin{tikzpicture}
        \begin{axis}[
            view={30}{65},
            hide axis,
            axis lines=none,
            width=9cm, height=7cm,
            z buffer=sort,
            xmin=-3.5, xmax=3.5,
            ymin=-1, ymax=4,
            zmin=-2, zmax=1,
            colormap={bw}{gray(0cm)=(1); gray(1cm)=(1)},
            mesh/interior colormap name=bw
        ]

        %
        \addplot3 [
            surf,
            shader=faceted interp,
            fill=white,
            draw=black,
            thin,
            domain=-3.5:3.5,
            domain y=-1:4,
            samples=50,    
            samples y=25,  
        ] 
        { -1 + (1 - 3.5 * x^2 * exp(-0.2*(y-1.5)^2)) * exp(-0.35*x^2) };

        \addplot3 [
            only marks,
            mark=*,
            mark size=3.5pt,
            mark options={fill=black}
        ] coordinates {(0, 1, 0)};
        
        \draw[->, >=Latex, line width=1pt, black] 
            (axis cs: 0, 1, 0) -- (axis cs: 0, 2.8, 0);

        \draw[->, >=Latex, line width=1pt, red!80!black] 
            (axis cs: 0, 1, 0) -- (axis cs: -1.0, 1, 0);

        \draw[->, >=Latex, line width=1pt, blue!80!black] 
            (axis cs: 0, 1, 0) -- (axis cs: 1.0, 1, 0);

        \end{axis}
    \end{tikzpicture}
    \caption{Local alignment instability under fine-tuning dynamics.
    An illustrative loss landscape showing a fine-tuning trajectory (black arrow) evolving near a ridge separating basins of low alignment utility. Although the initial gradient direction is nearly tangent to the ridge, induced acceleration toward a high-curvature alignment-sensitive direction (red/blue arrow), leads to departure.}
\end{figure}

\textbf{The Paradox of High-Dimensional Orthogonality.}
Empirical evidence reveals a striking pattern: alignment appears to reside in a small subspace of the parameter space.
\cite{wei2024assessing} demonstrates that pruning just $3\%$ of specific weights triggers catastrophic safety collapse, while removing similar proportions of weights from other parts of the model leaves its capabilities intact.
We call these directions \emph{alignment-sensitive}: parameter updates along them degrade safety, while updates in orthogonal directions should preserve it.
This suggests alignment is governed by a low-dimensional set of sensitive directions. 
In a model with billions of parameters, standard geometric intuition predicts that fine-tuning updates for unrelated tasks should be nearly orthogonal to these directions, thus, safe~\cite{arturi2025shared,li2018measuring}. 
Yet degradation occurs frequently across diverse tasks and model families, exhibiting consistent patterns~\cite{betley2025emergent,he2024what,qi2023fine}.

Existing theoretical frameworks do not fully resolve 
this paradox. 
%
Approaches based on data overlap~\cite{he2024what,hsiung2025your,yang2025alleviating} cannot fully explain why superficially similar tasks have vastly different safety impacts.
First-order geometric constraints assume static sensitivity subspaces and fail to account for how fine-tuning trajectories evolve over time~\cite{arturi2025shared,rajabi2024enhancing}.
Null-space methods~\cite{niu2025mitigating,zhang2025guardrail} presuppose we can identify and avoid alignment-sensitive directions but offer no explanation for why avoidance fails in practice.

\textbf{Our Contributions.}
We resolve this paradox by developing a theory of alignment degradation grounded in the curvature of trajectories in parameter space.
Our key insight is that while initial fine-tuning updates may indeed be orthogonal to alignment-sensitive
directions, this orthogonality is unstable under the dynamics of gradient descent.
The curvature of the fine-tuning objective\textemdash{}captured by second-order terms in the loss landscape\textemdash{}systematically bends trajectories into alignment-sensitive subspaces, even when first-order updates point away from them.

We formalize this mechanism through the \textbf{\underline{A}lignment \underline{I}nstability \underline{C}ondition (AIC)}, which identifies three interacting structural properties arising from the model's geometry and the fine-tuning task.
\begin{enumerate}
    \item \textbf{Low-Rank Sensitivity:} alignment for a given skill (or capability) concentrates in a small subspace with sharp curvature, characterized by the leading eigenvectors of the Fisher information matrix.
    \item \textbf{Initial Orthogonality:} fine-tuning gradients for unrelated tasks have negligible first-order projection onto this subspace, creating the illusion of safety.
    \item \textbf{Curvature Coupling:} second-order terms in the gradient flow induce nontrivial acceleration into the alignment-sensitive subspace, forcing trajectories to drift despite initial orthogonality.
\end{enumerate}
Under the AIC, we prove that alignment degradation is not merely possible, but structurally inevitable within the trust region of the aligned model.
Our main theoretical results are:
\begin{theorem}[Informal Version of Theorem~\ref{thm:util-bound}]
    Movement in the alignment-sensitive subspace $M_i\subseteq \mathbb{R}^n$ incurs quadratic utility loss, $\Omega(\lambda\|P_i(\Delta\theta)\|^2)$ where $\lambda$ is the minimum curvature in $M_i$ and $P_i$ projects onto this subspace.
\end{theorem}
\begin{theorem}[Informal Version of Theorem~\ref{thm:projection}]
    Along gradient flow trajectories, $\theta(t)$ where $\theta(0) = \theta^*$, satisfying the AIC, projection on $M_i$ grows quadratically: $\|P_i(\theta(t) - \theta^*)\| = \Omega(\gamma t^2)$, where $\gamma$ quantifies curvature coupling.
\end{theorem}
The two main theorems combine to yield a quartic scaling law for alignment degradation within our utility-based framework (defined in Section~\ref{sec:model}).
\begin{corollary}[Informal Version of Corollary~\ref{cor:quartic_onset}]
    Combining geometric loss with quadratic drift yields alignment degradation scaling as $\Omega(\lambda \gamma^2 t^4)$ in early training.
\end{corollary}

The quartic scaling law provides a principled explanation for rapid, seemingly sudden safety failures observed in practice~\cite{fraser2025fine,peng2024navigating,qi2024evaluating,wei2024assessing}.
Even small amounts of curvature coupling ($\gamma$) compound dramatically through the fourth power, causing trajectories to escape safety basins faster than first-order analysis would predict.

\textbf{Implications.}
Our theory reveals that alignment fragility is not an artifact of adversarial attacks or fixable through better data curation. 
%
Rather, it is an intrinsic geometric property of fine-tuning in high-dimensions. The AIC explains why: (1) benign tasks break safety, (2) low-rank adaptations fail, and (3) first-order defenses are insufficient.
Our results necessitate a fundamental rethinking of alignment preservation strategies.
In Section~\ref{sec:exp}, we empirically validate these predictions, showing that our notion of geometric overlap provides a principled explanation for degradation of alignment.

\section{Related Work}

Our work is positioned at the intersection of several research domains: the geometric analysis of neural network parameter spaces, the empirical study of alignment fragility, and the theoretical frameworks of multi-task and meta-learning.
We here discuss related studies which were not covered in the main text.

\paragraph{Geometric and Manifold-Based Approaches.}
The conceptualization of neural network representations and parameter spaces through a geometric lens is well-established~\cite{amari1997information,sagun2017empirical,thiruthummal2024information}.
This line of inquiry is broadly motivated by the manifold hypothesis~\cite{rifai2011manifold}, which posits that high-dimensional data, such as inputs for unsupervised learning or distinct classes for classification, concentrates near low-dimensional manifolds separated by regions of low density.
More recently, this geometric perspective has been directly applied to the process of fine-tuning~\cite{mantri2024rethinking}.
Specifically, Low-Rank Adaptation (LoRA) has been theoretically framed as a problem of manifold geometry~\cite{jiangloram}.
This aligns with foundational work in numerical optimization that analyzes fixed-rank matrices as an embedded geometry, even outside the context of model fine-tuning~\cite{rifai2011manifold}.

While these approaches provide valuable intuition, they often treat alignment geometry as static.
Recent work has attempted to identify null-spaces to the alignment weight subspace in an effort to constrain fine-tuning~\cite{zhang2025guardrail}.
However, these methods typically assume a flat local geometry that does not account for the empirical fragility of alignment in high-dimensions\textemdash{}a gap our work addresses by integrating this geometric view with a probabilistic analysis of curvature-induced drift.

\paragraph{Alignment Drift and Critical Parameters.}
Empirical research confirms that fine-tuning modifies only a small subset of the representational subspace, presumably preserving the pre-trained model's broad capabilities.
However, this preservation is asymmetric: while general capabilities may persist, aligned guardrails are uniquely brittle.
\cite{hsiung2025your} demonstrate that this degradation is systematic, showing that the impact of fine-tuning varies dramatically depending on the specific task structure, a phenomenon further categorized in broader surveys of trustworthy AI~\cite{liu2024safety}. 
Critically, this degradation occurs even on seemingly benign tasks such as solving math word problems~\cite{cobbe2021training} or standard instruction following~\cite{guan2025benign,he2024what, qi2023fine} where seemingly no adversarial intent exists. 

More critically,~\cite{das2025alignguard} introduce the ``DriftCheck'' dataset, demonstrating that even minor LoRA updates induce significant alignment drift. 
%
In fine-tuning with this dataset, the identify alignment critical parameters that are most relevant to maintaining safety behavior, finding that an overlap between the fine-tuning data and safety distributions accelerates this drift by overwriting fragile weight regions.
This observation connects directly with the safety relevant weights hypothesis of~\cite{wei2024assessing}, which assesses the brittleness of safety alignment via pruning, showing that removing specific weights triggers rapid safety collapse.
\citep{he2024what} and \citep{guan2025benign} further demonstrate that fine-tuning on benign data can compromise safety even without adversarial intent.
Our proposal seeks to explore the theoretical mechanism, which we refer to as the \textbf{Alignment Instability Condition}, that explains why these critical parameters are susceptible to drift even when the ecosystem of open-model development is expanding rapidly~\cite{osborne2024ai}.

\paragraph{Theoretical Frameworks for Skill Emergence and Adaptation.}
Our decomposition of alignment into distinct utility functions is grounded in recent theoretical advances regarding skill emergence in large language models. 
\cite{arora2023theory} formally model this phenomenon, proposing that complex skills emerge from the scaling of random feature combinations that allow the model to distinguish between increasingly granular tasks. 
This skill-based view is further supported by~\cite{malladi2023kernel}, who provide a kernel-based analysis of how fine-tuning modifies these underlying feature mappings. 
Our work adopts this foundational perspective by viewing alignment features not as global constraints but as distinct, emergent skills.
However, we reframe the problemfrom acquisition to preservation. 
While~\cite{arora2023theory} explain how these high-dimensional skills arise, our geometric theory addresses the inverse problem: why specific skills possess a unique structural fragility that makes them susceptible to catastrophic forgetting even when general capabilities remain intact.
This fragility is further understood by examining the intersection of our work with meta-learning theory. 
The standard paradigm for multi-task learning, as established in~\cite{du2020few,tripuraneni2021provable}, posits that diverse tasks can be solved by learning a single, shared low-dimensional representation, upon which task-specific linear heads are trained. 
In this idealized setting, fine-tuning is safe because the underlying representation is fixed. 
However, practical fine-tuning violates this assumption by updating the representation itself. \cite{chua2021fine} address this by relaxing the requirement to an approximate shared representation, where task-specific parameters are constrained to remain within a small distance of the central meta-parameters.
While~\cite{chua2021fine} prove that this relaxation allows for effective adaptation, our AIC reveals a critical blind spot in this framework regarding safety. 
We demonstrate that on curved manifolds, the allowable deviation cannot be isotropic. 
Due to the high curvature associated with alignment skills, even valid updates that satisfy the distance constraint for a new task can induce large, non-linear drifts in the safety subspace. 
Thus, our work challenges the sufficiency of the approximate shared representation hypothesis for safety-critical systems, showing that geometric curvature couples benign task updates with more harmful alignment degradation in ways that standard meta-learning bounds do not predict.

\section{Model \& Preliminaries} \label{sec:model}
We consider a language model parametrized by $\theta \in \Theta \subseteq \mathbb{R}^n$, defining a conditional distribution, or policy, $\pi_\theta(y | x)$ over responses $y \in \cY$ given prompts $x \in \cX$.
Unless otherwise stated, we let $\| \cdot\| $ denote the Euclidean norm of a vector and $\| \cdot\|_{\text{op}}$ denote the operator norm.

\subsection{Alignment as Skill Utility}
A key obstacle to theoretical analysis is that alignment is not monolithic.
A model simultaneously maintains multiple capabilities that can degrade independently during fine-tuning.
We therefore decompose alignment into distinct \emph{skills} (or capabilities), following recent work on skill emergence in language models~\cite{arora2023theory,malladi2023kernel}.

\begin{definition}[Alignment Skill and Utility]
    Let $\mathcal{S} = \{S_1, ..., S_m\}$ be a set of alignment skills.
    Each skill $S_i$ is associated with a reference distribution $\cD_i$ over interaction pairs $(x,y)$.
    We define the \textbf{utility} of the model $\theta$ for skill $S_i$,
    $$u_i(\theta) := \E_{(x,y) \sim \cD_i}\left[ \log \pi_\theta (y | x) \right].$$
\end{definition}
For a fixed base model, $\theta^*$, we define the utility loss (or alignment degradation) for skill $S_i$ under parameter update $\Delta\theta$ as $\Delta u_i(\theta) := u_i(\theta^*) - u_i(\theta)$.
This definition grounds model alignment in probabilistic terms: a model is well-aligned for skill $S_i$ if it assigns high likelihood to responses deemed ideal under $\cD_i$.

\subsection{Local Geometry and Fisher Information}
To analyze alignment preservation geometrically, we adopt an idealized but analytically convenient assumption that the base model is well-optimized for the alignment skills under consideration.

\begin{assumption}[Skill Optimality]
\label{asm:optimality}
The base model $\theta^*$ is optimal for every alignment skill $S_i$:$\pi_{\theta^*}(y|x) = \mathcal{D}_i(y|x)$
for all $x \in \supp(\cD_i)$ and all $y$ with $\cD_i(y|x) > 0$. 
\end{assumption}
This  assumption states that prior to fine-tuning, the base model exactly satisfies the developer-specified alignment objectives.
While idealized, similar well-defined assumptions are standard in theoretical analyses to obtain clean local characterizations of learning dynamics.
In our setting, Assumption~\ref{asm:optimality} ensures that $\theta^*$ is stationary for each skill's utility and that local alignment degradation admits a transparent reduction to a KL divergence~\cite{cover1999elements}.
We emphasize that Assumption~\ref{asm:optimality} is stronger than necessary and is adopted for clarity of presentation.
Our main results extend under weaker local conditions (see Appendix~\ref{sec:relax}).

\begin{lemma}[Alignment Loss as KL Divergence]
\label{lem:degradation_is_kl}
Given Assumption \ref{asm:optimality}, after fine-tuning to obtain model $\theta$, the degradation in utility for any skill $S_i$ is equal to the expected KL divergence between the policies, measured over the distribution $\mathcal{D}_i$:
$$ \Delta u_i = \E_{x \sim \mathcal{D}_i(x)}[\KL\left(\pi_{\theta^*}(y|x) || \pi_{\theta}(y|x)\right)]. $$
\end{lemma}
The proof follows from the definition of utility and is deferred to Appendix~\ref{sec:omit}.
This result highlights that alignment stability is determined by the model's policy divergence on the \emph{alignment} distribution $\mathcal{D}_i$, which may differ significantly from the fine-tuning distribution.
In practice, alignment comprises multiple skills (safety refusal, medical accuracy, factual correctness, etc.).
Our analysis focuses on individual skill degradation, where overall alignment fails when any critical skill degrades below acceptable thresholds.

\textbf{Trust Regions and Local Quadratic Geometry.}
A common safeguard, such as those in PPO~\cite{schulman2017proximal} or TRPO~\cite{schulman2015trust}, constrain updates within a KL trust region around the base policy:
$$D_{KL}(\pi_{\theta^*} || \pi_{\theta}) \le \Delta,$$
where $\Delta > 0$ is a small positive constant.
While such constraints were initially designed for optimization stability, in alignment settings they are often interpreted as preventing harmful deviation from a safe reference model~\cite{yang2024asymptotics}.
Geometrically, this restricts fine-tuning to a neighborhood where the parameter space is locally quadratic.
We quantify this curvature using the Fisher Information Matrix (FIM).
Let $\mbf{F}_x(\theta^*) = \mathbb{E}_{y \sim \pi_{\theta^*}(\cdot | x)}\left[ \nabla_\theta \log \pi_{\theta^*}(y|x) \nabla_\theta \log \pi_{\theta^*}(y|x)^\top \right]$ be the Fisher matrix for input $x$.
We further define the FIM specific to skill $S_i$ as
$\mbf{F}_i(\theta) := \mathbb{E}_{x \sim \cD_i}[\mbf{F}_x(\theta)].$
Assuming the model is $C^3$ smooth, the second-order Taylor expansion of the KL divergence yields the following local form for alignment loss:
\begin{proposition}[Local Geometric Form] \label{prop:local-geo}
        For sufficiently small displacements $\Delta\theta = \theta - \theta^*$,
        \begin{equation}
            \Delta u_i =  \frac12 \Delta\theta^\top \mbf{F}_i(\theta^*) \Delta\theta + O(\|\Delta\theta^3\|).
        \end{equation}
\end{proposition}
Thus, in the small-step regime enforced by KL trust regions, skill degradation is locally dominated by the projection of the update $\Delta\theta$ onto the high-curvature directions of $\mbf{F}_i(\theta^*)$. 

\subsection{Low-Rank Sensitivity Structure}
The key geometric insight is that not all directions in parameter space affect alignment equally.
Empirical studies of neural network loss landscapes reveal that Hessians (and FIMs) typically have spiked spectra: a few large eigenvalues dominate, while most directions have near-zero curvature~\cite{sagun2017empirical}. 
%
Recent empirical work also shows that model weights share a low-rank joint subspace across different tasks~\cite{kaushik2025universalweightsubspacehypothesis}.
This further motivates a low-rank decomposition of alignment sensitivity.

\begin{definition}[Alignment Sensitivity Subspaces] \label{def:subspace}
    The alignment sensitivity subspace $M_i$ for skill $S_i$ is the span of the $d$-leading eigenvectors of $\mbf{F}_i(\theta^*)$.
\end{definition}
Stemming from this definition, let $P_{i} : \mathbb{R}^n \rightarrow M_i$ denote the Euclidean orthogonal projection onto $M_i$, and $P_{i}^\top = I - P_{i}$, the projection onto its orthogonal complement.
Within our framework, directions in $M_i$ are critical for alignment, with small movements incurring large utility loss due to high curvature, whereas directions in the null space are locally benign.
This decomposition formalizes the empirical observations that alignment resides in small subspaces~\cite{arturi2025shared,li2018measuring,wei2024assessing}.



\section{Why Flat Geometry Fails} \label{sec:flat}
Section~\ref{sec:model} established a local geometric picture: utility loss is controlled by projections onto low-dimensional sensitivity subspaces, and KL trust regions restrict fine-tuning to neighborhoods where this quadratic approximation holds.
This suggests a powerful alignment-preservation guarantee: in high dimensions, unrelated fine-tuning updates should be nearly orthogonal to any fixed low-rank subspace, making degradation exponentially unlikely.
We formalize this intuition as a probabilistic \textbf{null model}
, then contrast it against empirical evidence showing that real fine-tuning trajectories violate these predictions systematically. 


\subsection{Benign Fine-Tuning in the Flat Regime}
To model benign fine-tuning without additional structure, we idealize the update direction as a random unit vector in parameter space.

%

\begin{theorem}[Benign Fine-Tuning under Flat Geometry] \label{thm:random}
    Fix an alignment skill $S_i$ and assume there exists $r > 0$ such that for all $\|\Delta\theta\| \le r$, we have that $\Delta u_i = \frac12\Delta\theta^\top \mbf{F}_i(\theta^*) \Delta\theta + O(\|\Delta\theta\|^3)$.
    Let $g \sim \text{Unif}(\mathbb{S}^{n-1})$ be a random unit vector, and consider the update $\Delta\theta = \eta g$ where $\eta > 0$ is sufficiently small. Then: (1) the expected
    utility loss satisfies
    $$\mathbb{E}_{g \sim \text{Unif}(\mathbb{S}^{n-1})}[\Delta u_i] = \frac{\eta^2}{2n}\text{Tr}\left(\mbf{F}_i(\theta^*)\right) + O(\eta^3),$$
    (2) for any rank-$d$ orthogonal projector $P$, we have $\|P(g)\|^2 = \Theta(\sfrac{d}{n})$, with probability at least $1-2e^{-c\epsilon^2d}$.
    This further implies that
    $$\Delta u_i = O \left( \frac{d\eta^2}{n}(1+\epsilon) \lambda_{\max} (\mbf{F}_i(\theta^*)) \right).$$
\end{theorem}
The proof uses rotation invariance of the uniform distribution on the sphere and concentration of $\chi^2$ random variables (Appendix~\ref{sec:omit}).

%
This theorem formalizes the high-dimensional safety intuition:
when the utility loss is well-approximated by a fixed quadratic form and fine-tuning updates are unstructured, the expected degradation scales with the average curvature $\text{Tr}(\mbf{F}_i(\theta^*))/n$ and the probability of significant overlap with any fixed low-dimensional sensitivity subspace is exponentially small in its dimension. 
In this regime, alignment degradation is suppressed both in expectation and with high probability by the ambient dimensionality of the parameter space. 
Consequently, if harmful or alignment-critical directions are confined to a small subspace and fine-tuning updates are effectively random, narrow fine-tuning should be benign. 
%
%
However, this conclusion rests on two idealizations: (1) updates are unstructured (modeled as random), and (2) the Fisher geometry is static throughout training. 
As we show next, both assumptions fail in practice.
%
%

\subsection{Empirical Alignment Drift}
A growing body of evidence contradicts the null model's predictions.
Alignment degradation is not rare, and is, in fact, common, rapid and strongly dependent on task structure in ways not explained by random projection.

\textbf{Benign Tasks Break Safety.}
The strongest challenge to the predictions of Theorem~\ref{thm:random}
comes from fine-tuning on datasets with no harmful content.
\citep{he2024what} demonstrate that fine-tuning LLaMa-2-7B and GPT-3.5-Turbo on standard instruction datasets for fewer than 100 steps causes sharp drop off (within five gradient steps) in safety refusal on adversarial prompts, consistent with the results of~\citep{fraser2025fine}.
Their work further shows that even carefully filtered data can trigger such jailbreaks with only 100 benign instruction-following samples.
\citep{betley2025emergent} and \citep{qi2023fine} further discuss this ``emergent misalignment'' by showing that narrow task adaptation systematically erodes safety boundaries. 

\textbf{Non-Uniform Sensitivity and Threshold Effects.}
If degradation arose from random overlap as Theorem~\ref{thm:random} predicts, we would expect smooth, gradual decline proportional to training time.
Instead, \citep{wei2024assessing} observe sharp transitions, where modifying a specific $3\%$ of weights causes catastrophic safety collapse, while random $3\%$ samples have negligible effect.
This non-uniformity directly contradicts the null model's assumption that degradation scales with average curvature across all parameters.
Moreover, \citep{qi2023fine} and \citep{turner2025model} document threshold behaviors where models remain stable for several steps before rapid failure onset, again inconsistent with smooth probabilistic degradation.

\textbf{Task-Dependent Variation.}
The random-update null model of Theorem~\ref{thm:random} further implies that utility loss should be predictable from data overlap: tasks with distributions similar to $\cD_i$ should cause degradation (overlapping data would induce updates with non-negligible projection onto $M_i$) while dissimilar tasks should be safe.
%
However, \citep{hsiung2025your} show that task influence is highly variable and poorly predicted by natural metrics.
\citep{he2024what} formalize this by demonstrating that gradient-based selection finds harmful subsets within benign data---not by simple similarity-based computations.

All of these phenomena reveal a consistent signature for alignment degradation: it is systematic, it is rapid, and it is task-dependent (not similarity based).

\subsection{Failure of Static Geometry}
Our key observation is that the cause of the null model failing to explain observed empirical behaviors is its assumption that $\mbf{F}_i(\theta)$ is static, i.e., the sensitivity subspace $M_i$ remains fixed throughout fine-tuning.
Such an assumption is, in fact, only valid in a vanishingly small neighborhood of $\theta^*$.
In reality, both the magnitude and orientation of the Fisher information matrix evolve as parameters change along the trajectory $\theta(t)$.
Thus, directions that were initially safe become dangerous.
Moreover, fine-tuning updates themselves are highly structured in accordance with the task objective, and thus not random. 

\section{Alignment Instability Condition} \label{sec:aic}
Section~\ref{sec:flat} demonstrated that static, first-order analysis cannot explain empirical alignment degradation.
The resolution requires understanding how fine-tuning trajectories evolve in a dynamic geometry where both update direction and the sensitive structure change over time.

We observe that in practice, fine-tuning induces a continuous trajectory $\theta(t)$ through parameter space along which the Fisher information varies: its spectrum shifts, its eigenspaces rotate, and the alignment sensitivity subspace $M_i$ itself evolves.
As a result, directions initially orthogonal to $M_i(\theta^*)$ need not remain, and small first-order  projections can accumulate into significant degradation through higher-order effects. 
This section formalizes the precise structural conditions under which such accumulation becomes inevitable.

\subsection{Intuition Behind Two Degradation Regimes}
Let $g(\theta) = \nabla L_{\textsc{ft}}(\theta)$ denote the fine-tuning gradient and $M_i(\theta)$ the sensitivity subspace for skill $S_i$ at parameters $\theta$.
We distinguish two fundamentally different mechanisms of alignment degradation based on the initial overlap (measured by Euclidean projection) between the update direction and sensitive subspace.

\textbf{First-Order Regime.}
If the initial fine-tuning gradient has substantial projection onto $M_i$, that is, $\|P_i(\theta^*)g(\theta^*)\| \ge c$ for some non-negligible constant $c > 0$, then alignment degradation begins immediately.
By the gradient flow $\dot{\theta}(t) = -g(\theta(t))$ with $\theta(0) = \theta^*$ we readily obtain that $\|P_i(\theta^*)(\theta(t) - \theta^*)\| = \Theta(ct)$.
Combining with Proposition~\ref{prop:local-geo} yields quadratic degradation from first-order overlap.
This regime corresponds to the more trivial fine-tuning tasks that directly modify alignment-relevant behaviors, requiring no additional geometric analysis beyond first-order projections.
The remainder of this section focuses on the more subtle second-order regime.

\textbf{Second-Order Regime.}
Consider instead the scenario where the initial gradient is nearly orthogonal to the sensitivity subspace: $\|P_i(\theta^*)g(\theta^*)\| < \epsilon$ for small $\epsilon > 0$.
Here, the first-order term vanishes and, naively, the trajectory should remain safe.
However, this conclusion assumes that $M_i$ and $g$ remain fixed.
When accounting for how such quantities evolve, second-order terms dominate.
Specifically, expanding $\theta(t)$ along the gradient flow to second order, we obtain:
$$\theta(t) - \theta^* = -tg(\theta^*) + \frac{t^2}{2}\nabla g(\theta^*)g(\theta^*) + O(t^3).$$
Under the projection, the first term contributes only $\epsilon t$, whereas the second term\textemdash{}the directional derivative of the gradient\textemdash{}can have substantial projection onto $M_i$.
This acceleration term captures how the curvature of the fine-tuning loss bends the trajectory.
%

\subsection{Geometric Intuition: Curvature as Steering}
If parameter space were naturally Euclidean and $\mbf{F}_i(\theta)$ were constant, gradient flow would proceed in straight lines and a trajectory starting perpendicular to $M_i$ would remain perpendicular, preserving alignment.
In reality, the information geometry of neural networks is non-Euclidean~\cite{amari1997information,thiruthummal2024information}, with distance between distributions depending upon the position.
Crucially, as the model adapts through parameter space, the metric changes, the subspaces rotate, and the gradient curves.

\textbf{Curvature Coupling.}
The key quantity, $\nabla g(\theta^*) g(\theta^*)$, is the directional derivative of the fine-tuning gradient.
Geometrically, this is the acceleration of the trajectory, or, how much the direction of steepest descent changes.
Most important for the present study, even if $g(\theta^*)$ points away from $M_i$, the derivative $\nabla g(\theta^*)$ can point toward it.
Through fine-tuning, we then have three interacting mechanisms
dictating the misalignment process: $M_i$ must have high curvature so movement into it damages safe behaviors (this determines loss magnitude, analogous to the first-order regime), the orthogonality of a trajectory to this subspace dictates potential movement within this space (initial condition of the second-order regime), and the non-trivial coupling of $\nabla g(\theta^*) g(\theta^*)$ can accelerate projection into this space (the acceleration mechanism unique to second-order regime).
%

\subsection{Formal Definition}
We now formalize the structural condition under which curvature forces trajectories into alignment-sensitive directions despite initial orthogonality.
Throughout, for analytical clarity and without loss of generality, we assume that $\mbf{F}_i(\theta)$ and $g(\theta)$ are twice continuously differentiable in a neighborhood of $\theta^*$.

\begin{definition}[Alignment Instability Condition] \label{def:aic}
    Fix a skill $S_i$.
    Let $\mbf{F}_i(\theta^*)$ have eigenvalues $\lambda_1 \ge \lambda_2 \ge \ldots \ge \lambda_k \ge 0$, and let $M_i(\theta^*)$ be the span of the top $d$-eigenvectors of $\mbf{F}_i(\theta^*)$, with $P_i(\theta^*)$ the orthogonal projection onto $M_i(\theta^*)$ 
    We say that $\theta^*$ satisfies the
    \textbf{Alignment Instability Condition (AIC)} for skill $S_i$ with parameters $(d,\lambda,\gamma,\epsilon)$ if the following hold:
    \begin{enumerate}
        \item \textbf{Low-Rank Sensitivity:} The Fisher spectrum at $\theta^*$ is concentrated in a $d$-dimensional subspace: $\sum_{j>d}\lambda_j \le \epsilon$ and $\lambda_d \ge \lambda$.
        \item \textbf{Initial Orthogonality:} The fine-tuning update direction is initially nearly orthogonal to the sensitivity subspace: $\|P_i(\theta^*)g(\theta^*)\| \le \epsilon$.
        \item \textbf{Curvature Coupling:} The curvature of the fine-tuning objective induces nontrivial acceleration into the sensitivity subspace: 
        $$\|\mbf{F}_i(\theta^*)^{1/2}P_i(\theta^*)\nabla g(\theta^*) g(\theta^*)\| \ge \gamma.$$
    \end{enumerate}
\end{definition}
At a high level, the AIC identifies the regime where alignment appears stable initially (Condition 2) yet is structurally doomed (Condition 3) because it resides in a sharp, low-dimensional region (Condition 1).
More formally, Condition 1 isolates the small set of directions where a skill $S_i$ is fragile, with the bounds ensuring sufficient curvature concentrated in the first $d$ eigenvectors and the tail contributing negligibly.
Condition 2 captures the natural defense measure of projection away from sensitive subspaces, like that of \citep{zhang2025guardrail}.
%
However, we note that our analysis in Section~\ref{sec:aic-theory} reveals this is a snapshot solution: unless projections are continuously updated to track the evolving sensitive subspace, curvature coupling (Condition 3) inevitably bends the trajectory back into sensitive regions over time.
More specifically,this captures the rotational geometry crux: as we move along the trajectory, the gradient rotates toward $M_i$, with the parameter $\gamma$ quantifying the strength of coupling.



\section{Consequences of AIC} \label{sec:aic-theory}
We now develop the main theoretical consequences of the AIC, establishing that alignment degradation under these conditions is not merely possible but structurally inevitable.
These results enable both qualitative understanding (why benign tasks break safety) and quantitative prediction (how fast degradation occurs).
Throughout this section, we fix a skill $S_i$ and assume that $\theta^*$ satisfies Definition~\ref{def:aic} with parameters $(d,\lambda,\gamma,\epsilon)$.
Our approach proceeds in two steps, mirroring the geometric intuition from Section~\ref{sec:aic}.
We first quantify how much utility degrades when the trajectory enters the sensitive subspace $M_i$.
Subsequently, we prove that the AIC forces the trajectory to enter $M_i$ at a quadratic rate.
Combining the two results, we derive the quartic scaling law.
The key insight is that these are separate geometric phenomena that when composed yield the quartic scaling law. 

For clarity of presentation, the arguments in this section are carried out with respect to the sensitive subspace defined by the Fisher information at $\theta^*$.
We note, however, that the conclusions are robust to local variation of the Fisher information and to rotation of the sensitive subspace along the trajectory; accounting for these effects introduces only constant-factor and lower-order corrections and is deferred to Appendix~\ref{sec:fisher}.

\subsection{Projection of Misalignment Inevitability}
We begin by formalizing the intuition that alignment resides in a small number of sharp directions.
Movement into the subspaces $M_i$ then degrades skill $S_i$, with loss growing quadratically in the projection magnitude.

Recall from Section~\ref{sec:model} that utility loss locally takes the form $\Delta u_i \approx \frac12 \Delta\theta^\top \mbf{F}_i(\theta^*) \Delta\theta$.
Here, the Fisher matrix characterizes the local curvature of alignment utility, identifying directions along which small parameter changes induce large distributional shifts.
AIC Condition 1 concentrates the eigenspectrum, implying that movement in $M_i$ causes this divergence to scale with $\lambda$ whereas movement outside is negligible.
Effectively, only the component projected into $M_i$ matters for utility.
The following theorem establishes the ``cost'' for entering sensitive regions.
\begin{theorem} \label{thm:util-bound}
    Under Assumption~\ref{asm:optimality} and assuming AIC holds, then there exists a neighborhood $U$ of $\theta^*$ such that for all $\theta \in U$ with update step $\Delta \theta = \theta - \theta^*$:
    $$\Delta u_i(\theta) \ge \frac{1}{2} \|F_i^{1/2}P_i(\theta-\theta^*)\|^2 - O\left(\|\Delta \theta\|\right)^3.$$
    Consequently, if the projection satisfies $\|P_i(\Delta \theta)\| \ge \delta$ and the step size $\|\Delta \theta\|$ is sufficiently small, we have
    $\Delta u_i(\theta) = \Omega(\lambda \delta^2).$
\end{theorem}
The proof follows from decomposing $\mbf{F}_i(\theta^*)$ and writing $\Delta\theta$ in the eigenbasis.
Projecting this change with $P_i$ then decomposes into the overlapping and orthogonal components, both of which can be bounded according to the AIC conditions.
Full details are presented in Appendix~\ref{sec:omit}.
Observe that this theorem characterizes the loss if projection occurs, but does not capture whether such overlap occurs during fine-tuning.
We now focus on showing that, under the AIC, drift into $M_i$ is dynamically forced.

\subsection{Curvature-Induced Drift}
Even when the initial fine-tuning gradient is nearly orthogonal to $M_i$, the curvature of the loss landscape forces the trajectory to drift into the sensitive subspace at a quadratic rate.
This drift is dictated not by the velocity ($-g(\theta^*)$), but the acceleration ($\nabla g(\theta^*) g(\theta^*)$) which monitors how the direction of steepest descent changes along a trajectory.
AIC Condition 3 bounds this factor, ensuring that projection $M_i$ increases through training.
Thus, our initially small overlap trajectory drifts into sensitive regions.
Note that this acceleration is due to the curvature of the fine-tuning objective. 
Gradient descent follows the local geometry with no mechanism to anticipate or counteract curvature-induced drift.

\begin{theorem} \label{thm:projection}
    Fix a skill $S_i$.
    Let $g(\theta) := \nabla L_{\textsc{ft}}(\theta)$ be the fine-tuning gradient and consider the gradient-flow trajectory defined by $\dot{\theta}(t) = -g(\theta(t))$ and $\theta(0) = \theta^*$.
    Assuming $g$ is $C^2$ in an $r$-ball of $\theta^*$ and the AIC holds, then there exists a $t_0>0$ such that for $t \in (0,t_0]$ we have that
    $$\|F_i^{1/2}P_i(\theta(t) - \theta^*)\| \ge \frac{\gamma}{2}t^2 - \epsilon t - O(t^3).$$
\end{theorem}
The proof follows from Taylor expanding the trajectory along the gradient flow to second order.
Projecting this trajectory onto $M_i$ via $P_i$, we can then bound the resulting terms using the AIC in combination with the reverse triangle inequality to complete the result.
Full details are in Appendix~\ref{sec:omit}.

\subsection{Combing Geometry and Dynamics}
Theorems~\ref{thm:util-bound} and~\ref{thm:projection} address complementary aspects of alignment degradation.
Composing these results yields our main prediction: alignment utility
degrades as the fourth power of training evolution.

\begin{corollary}[Quartic Onset of Alignment Degradation]
\label{cor:quartic_onset}
Fix a skill $S_i$ and let $\theta(t)$ be a fine-tuning trajectory satisfying the assumptions of Theorems~\ref{thm:util-bound} and~\ref{thm:projection}.
Then there exists $t_0>0$ such that for all $t \in [0, t_0],$
the utility loss for skill $S_i$ satisfies
$\Delta u_i = \Omega(\lambda\gamma^2\, t^4).$
\end{corollary}
The quartic law provides a quantitative explanation for rapid safety collapse observed in practice~\cite{he2024what,qi2023fine}. 
Thus, our results show that this perceived sudden failure is actually a smooth quartic due to curvature coupling.

We lastly note that the preceding results are stated with respect to the sensitive subspace $M_i$ defined by the Fisher information at
$\theta^*$. 
In practice, $F_i(\theta)$ varies along the fine-tuning trajectory, inducing rotation of the corresponding eigenspaces. 
Under mild local regularity, this rotation can be controlled via standard
Davis-Kahan perturbation bounds, and the proofs extend with only constant-factor and lower-order corrections to the drift
and loss bounds. 
We provide the formal statement and proof in Appendix~\ref{sec:fisher}.




\section{Experiments} \label{sec:exp}
Our theory establishes that alignment degradation is geometrically inevitable under the AIC, with the Fisher information matrix (FIM) characterizing the curvature structure that governs this fragility.
We now validate (1) that the FIM exhibits low-rank structure required by the AIC, and (2) that our notion of geometric overlap successfully predicts which fine-tuning tasks will degrade alignment.

\subsection{Low-Rank Fisher Structure}
Our theory operates in weight space, where the FIM's eigenvector define alignment-sensitive directions (Definition~\ref{def:subspace}).
However, prior empirical work characterizes alignment through low-rank structure in representation space (activations), rather than weight space~\cite{arditi2024refusal,wei2024assessing, zhang2025guardrail, zhang2026understanding}.
Specifically, these works show that for a linear layer $W \in \mathbb{R}^{d_{\text{out}} \times d_{\text{in}}}$, the average gradient $\nabla L_W(T)$ with respect to alignment sequences $T$ typically concentrates in an $r$-dimensional subspace of the representation space, with $r \ll \min\{d_{\text{out}}, d_{\text{in}}\}$.

This distinction is crucial because representation-space low-rank identifies which neurons are sensitive, while weight-space low-rank via FIM quantifies how much curvature exists along sensitive directions ($\lambda$ in our theory).
These are complementary perspectives on the same underlying geometry, but our theory relies on the FIM formulation to characterize second-order drift (Theorem~\ref{thm:projection}).
In Appendix~\ref{sec:omit-fim},
we prove that these two notions are related: if $\nabla L_W(T)$ concentrates in an $r$-dimensional representation space, then the FIM inherits low-rank, establishing AIC Property 1.
Thus, existing empirical evidence implies our weight-space 
Fisher structure, but we validate this directly.

\textbf{Computational Approach.}
Computing the full FIM for billion-parameter models is intractable.
Following~\citep{zhang2026understanding}, we use block-wise analysis, treating each module separately.
For a weight matrix $W \in \mathbb{R}^{d \times d}$, we first flatten to $\mathbb{R}^{d^2}$, then project to a low-dimensional space using Rademacher random projection~\cite{he2024what, johnson1984extensions,Trak,LESS}.
We then compute gradients on 100 safe completions from BeaverTails~\cite{ji2023beavertails} and form the projected FIM.

\textbf{Results.}
In Figure~\ref{fig:low_rank_fim}, we visualize sorted eigenvalues for $W_{\text{up}}, W_{\text{gate}}$ and $W_{\text{down}}$ across all MLP blocks (with the remaining modules deferred to Appendix~\ref{app:full_results}).
Notably, the eigenvalue decay exhibits the sharp low-rank structure leveraged by our theory.
This validates AIC Property 1 empirically and confirms that the connection from representation-space to weight-space low-rank holds in practice.

\begin{figure}[t]
    \centering
    \includegraphics[width=0.75\linewidth]{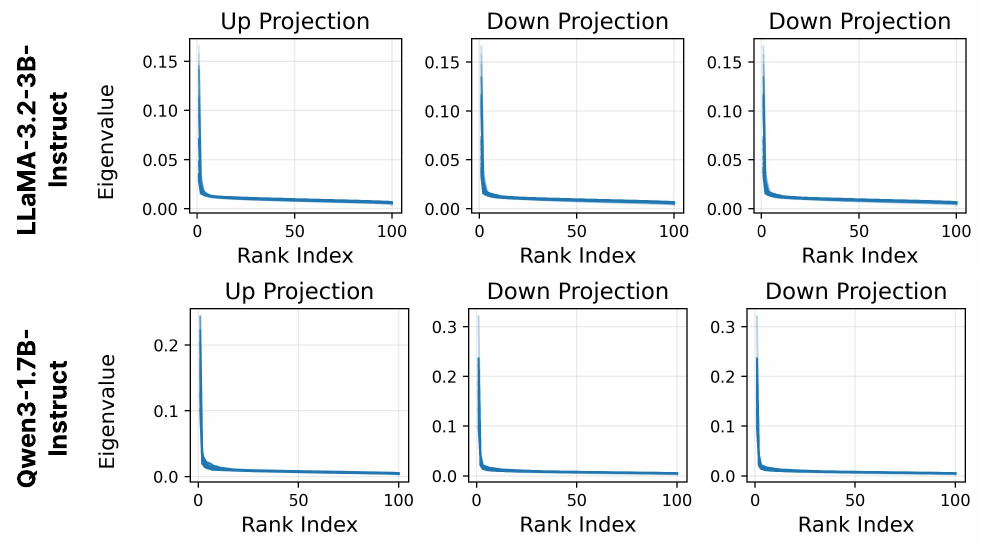}
    \caption{Top eigenvalues of FIM approximated over 100 random samples from \textit{BeaverTail}'s safe subset. Each subplot consists of multiple lines in different transparency levels for different layer indices, with each layer showing a similar low-rank structure.}
    \label{fig:low_rank_fim}
\end{figure}

\subsection{Geometric Overlap and Misalignment}
Recent work documents emergent misalignment (EM): fine-tuning on seemingly benign datasets causes safety degradation~\citep{betley2025emergent, turner2025model}. 
\citep{wang2025persona} explain this through ``persona shifts'' identified via Sparse Autoencoders (SAEs)~\citep{bricken2023monosemanticity, cunningham2023sparse,gao2025scaling} in representation space, a diagnostic approach that requires access to a target evaluation set to identify which personas are problematic. 
We demonstrate that our geometric framework offers a complementary and predictive explanation.
Specifically, EM arises from geometric overlap between fine-tuning updates and the Fisher-defined alignment subspace, quantifiable \textit{before} observing safety failures.

We approximate the utility loss from Proposition~\ref{prop:local-geo} using an \textbf{Overlap Score (OS)}. 
For each fine-tuned model, we extract per-module weight changes $\Delta W \in \mathbb{R}^{d \times d}$, flatten to $\mathbb{R}^{d^2}$, apply the same Rademacher projection $P_{\text{random}} \in \mathbb{R}^{4096 \times d^2}$ used for the FIM, and compute:
\[
\textbf{OS} = \big( P_{\text{random}} \text{vec}(\Delta W) \big)^\top \mathbf{F} \big(P_{\text{random}} \text{vec}(\Delta W)\big).
\]
This measures how much the fine-tuning update projects onto the high-curvature alignment subspace. 
Higher OS predicts greater alignment degradation, enabling prospective risk assessment.

\textbf{Experimental Setup.}
We fine-tune Qwen3-1.7B-Instruct~\cite{qwen3technicalreport} and LLaMA-3.2-3B-Instruct~\cite{grattafiori2024llama} using both LoRA~\cite{hu2022lora} and full fine-tuning on datasets in three categories: (1)~\textit{Benign} (SamSum, Alpaca Random 100, GSM8K Random 100), (2)~\textit{Seemingly Benign} (Alpaca Top 100 via gradient matching~\cite{he2024what}, Risky Financial Advice, Bad Medical Advice), and (3)~\textit{Harmful} (Pure Bad~\cite{qi2023fine}). 
Following standard protocols~\cite{he2024what, wang2024mitigating, zhang2025guardrail}, we evaluate using Gemini-2.5-Flash~\citep{comanici2025gemini} as a judge to assign Harmfulness Scores (HS, 1-5 scale) on 520 AdvBench queries~\cite{zou2023universal}.
Full dataset and training details in Appendix~\ref{app:experiment_details}.

Table~\ref{tab:harmfulness_eval} showcases that the three dataset categories produce distinct harmfulness levels, with \textit{Seemingly Benign} datasets achieving HS scores comparable to explicitly harmful data.
Critically, all fine-tuning increases harmfulness relative to the base model, but the magnitude is strongly modulated by dataset characteristics.
This confirms our framework's prediction that alignment degradation is unavoidable but scales with task-geometry coupling.
In full fine-tuning, our Overlap Score successfully distinguishes dangerous from safe tasks (Figure~\ref{fig:avg_overlap}). 
The \textit{Seemingly Benign} datasets exhibit OS patterns similar to \textit{Pure Bad}, while truly benign datasets show substantially lower overlap. 
This validates that geometric coupling plays a role in degradation: datasets with no apparent overlap with harmful content still break safety when they couple strongly to alignment-sensitive parameters. 
One exception is \textit{Alpaca Top 100}, where OS does not clearly separate from benign datasets, likely due to (1) noise in random projection to lower dimensions, or (2) Alpaca's diversity compared to domain-specific datasets.

\textbf{Limitations.} We additionally note that LoRA does not show the same clear OS-to-HS correlation (Figure~\ref{fig:avg_overlap}, right panels). 
We hypothesize this occurs because LoRA introduces ``intruder dimensions''~\citep{shuttleworth2025lora} that shift the model's top eigenvectors, making overlap with the pre-training Fisher less reliable. 
This supports our theoretical result on parametrization dependence. 
Additionally, the first-order OS may be insufficient for LoRA because its constrained update structure amplifies the induced coupling condition (AIC Condition 3), requiring second-order curvature analysis via $\nabla g(\theta^\ast)$. 
Computing this remains computationally prohibitive~\cite{kunstner2019limitations}; developing tractable curvature estimation is critical future work.

In spite of this limitation, we note that while \citep{wang2025persona}'s SAE-based approach identifies what changes in representations, our geometric framework explains \textit{why} these changes occur (curvature forces drift into $\mathcal{M}_i$) and predicts which tasks will cause them via OS. Critically, our approach requires only the fine-tuning dataset and pre-computed FIM\textemdash{}no access to target evaluation sets or trained SAEs\textemdash{}making it applicable for prospective risk assessment before deployment.

\begin{figure}[t]
    \centering
    \includegraphics[width=0.75\linewidth]{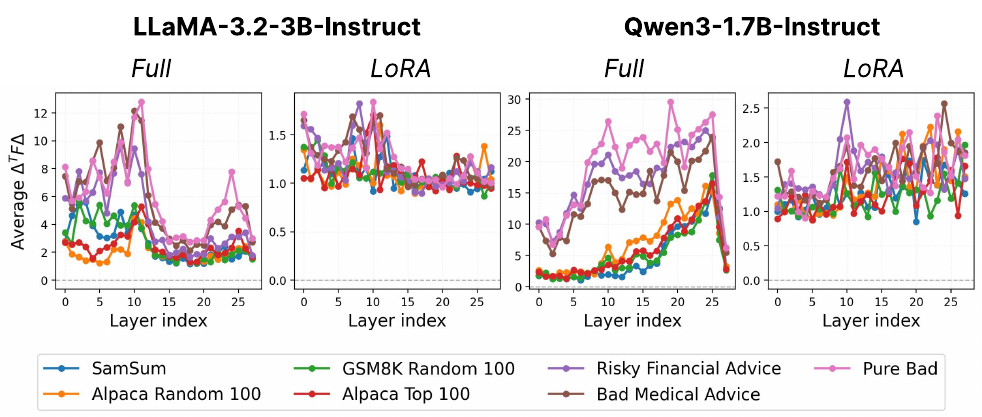}
    \caption{Average Overlap Score per Transformer Block of 7 Fine-Tuning Datasets. Each value represents the average OS of all modules in the block.}
    \label{fig:avg_overlap}
\end{figure}

\begin{table}[h]
    \centering
    \caption{Harmfulness Score (HS) across fine-tuning datasets}
    \resizebox{0.7\linewidth}{!}{
        \begin{tabular}{l *{4}{c}}
        \toprule
        \multirow{2}{*}{\textbf{Finetuning Dataset}} & \multicolumn{2}{c}{\textbf{Qwen3-1.7B-Instruct}} & \multicolumn{2}{c}{\textbf{LLaMA-3.2-3B-Instruct}} \\
        \cmidrule(lr){2-3} \cmidrule(lr){4-5}
         & \textbf{LoRA} & \textbf{Full} & \textbf{LoRA} & \textbf{Full} \\
        \midrule
        No Fine-tuning & \multicolumn{2}{c}{1.24} & \multicolumn{2}{c}{1.11} \\
        \midrule
        \textit{{\textbf{{Benign}}}} & & & & \\
        SamSum & 2.07 & 2.15 & 1.10 & 1.77 \\
        Alpaca Random 100 & 2.68 & 2.92 & 1.31 & 1.76 \\
        GSM8K Random 100 & 3.63 & 3.42 & 1.10 & 1.38 \\
        \midrule
        \textit{{\textbf{{Seemingly Benign}}}} & & & & \\
        Alpaca Top 100 & 3.15 & 3.68 & 1.18 & 3.70 \\
        Risky Financial Advice & 3.94 & 3.78 & 4.71 & 3.82 \\
        Bad Medical Advice & 3.92 & 3.86 & 2.46 & 3.65 \\
        \midrule
        \textit{{\textbf{{Harmful}}}} & & & & \\
        Pure Bad & 4.55 & 3.98 & 4.89 & 4.51 \\
        \bottomrule
        \end{tabular}
    }
    \label{tab:harmfulness_eval}
\end{table}

\section{Discussion \& Implications} \label{sec:end}



We have established that alignment degradation is an intrinsic geometric property of gradient-based fine-tuning, not an artifact of adversarial attacks or poor data curation. 
The AIC reveals why first-order defenses fail: while projecting updates into null spaces provides temporary protection, second-order curvature effects systematically bend trajectories back into sensitive subspaces. 
Our quartic scaling law explains the rapid onset of safety failures observed empirically and demonstrates that current defenses are fundamentally insufficient because they ignore the dynamic, curved geometry of the parameter space.
Moreover, our framework enables practical diagnostic tools for safer fine-tuning. We propose that practitioners can estimate the coupling parameter\ $\gamma$ before training to assess risk, monitor the projection size during training for early warning signals, and audit these values after training to quantify alignment degradation.
More fundamentally, effective preservation of alignment requires curvature-aware methods that track evolving sensitive subspaces, constrain second-order acceleration into these regions, and balance competing curvature constraints across multiple alignment skills.


Our results challenge the assumption that alignment preservation is primarily an engineering problem solvable through better data filtering. Instead, it is a fundamental question of geometric control in high-dimensional curved spaces, with profound implications for open-weight model deployment. The structural fragility we expose cannot be eliminated through pre-deployment alignment alone, requiring the community to reconsider whether current deployment paradigms are compatible with safety requirements.

\bibliographystyle{acm}
\bibliography{references}

\clearpage
\appendix

\section{Omitted Proofs} \label{sec:omit}

\subsection{Proofs from Section~\ref{sec:model}}
\begin{proof}[Proof of Lemma~\ref{lem:degradation_is_kl}]
    Recall the definition of alignment utility for skill $S_i$ is the expected log-likelihood under ideal distribution $\cD_i$:
    $$u_i(\theta) = \mathbb{E}_{(x,y) \sim \cD_i}\left[ \log \pi_\theta(y|x) \right]$$
    The alignmment loss is defined as $\Delta u_i(\theta) = u_i(\theta^*) - u_i(\theta)$.
    Substituting definitions, we obtain
    $$\Delta u_i = \mathbb{E}_{(x,y) \sim \cD_i}\left[ \log \pi_{\theta^*}(y|x) \right] - \mathbb{E}_{(x,y) \sim \cD_i}\left[ \log \pi_\theta(y|x) \right]$$
    By linearity of expectations we combine terms:
    $$\Delta u_i = \mathbb{E}_{x \sim \cD_i}\left[ \sum_y \cD_i(y |x) \log \left(\frac{\pi_{\theta^*}(y | x)}{\pi_\theta(y|x)}\right) \right]$$
    Invoking Assumption~\ref{asm:optimality}, we have that the aligned policy matches the ideal distribution, $\pi_{\theta^*}(y|x) = \cD_i(y|x)$ for all $x \in \text{supp}(\cD_i)$.
    Substituting $\cD_i$ with $\pi_{\theta^*}$, we obtain
    $$\Delta u_i = \mathbb{E}_{x \sim \cD_i}\left[ \sum_y \pi_{\theta^*}(y |x) \log \left(\frac{\pi_{\theta^*}(y | x)}{\pi_\theta(y|x)}\right) \right]$$
    where the term inside the expectation is precisely the KL divergence between $\pi_{\theta^*}$ and $\pi_\theta$.
\end{proof}

\begin{proof}[Proof of Proposition~\ref{prop:local-geo}]
    This follows from the second-order Taylor expansion of the KL divergence.
    The KL divergence, $\text{D}_{\text{KL}}(\pi_{\theta^*} || \pi_\theta)$ is a function of $\theta$ that is minimized when $\theta = \theta^*$.
    The gradient of the function at the minimum is zero and the Hessian is the Fisher information matrix.
    Expanding and taking the expectation over $x \sim \cD_i$ gives the desired result.
\end{proof}

\subsection{Proofs from Section~\ref{sec:flat}}
We first prove the isotropy of the uniform sphere for completeness.
\begin{lemma} \label{lem:iso}
    Let $g \sim \text{Unif}(\mathbb{S}^{n-1})$. Then $\mathbb{E}_g[gg^\top] = \frac1n I_n$
\end{lemma}
\begin{proof}
    Let $M := \mathbb{E}_g[gg^\top]$.
    For any orthogonal matrix $Q$, $Qg$ has the same distribution as $g$ by rotational invariance of $\text{Unif}(\mathbb{S}^{n-1})$.
    Hence,
    \begin{align*}
        M &= \mathbb{E}_g[gg^\top] \\
        &= \mathbb{E}_g[(Qg)(Qg)^\top] \\
        &= Q \mathbb{E}_g[gg^\top]Q^\top \\
        &= QMQ^\top
    \end{align*}
    for any such $Q$.
    Thus, $M$ commutes with orthogonal matrices.
    The only matrices with this property are scalar multiples of the identity, $M = cI_n$ for some $c \in \mathbb{R}$.
    Taking traces,
    $$\text{Tr}(M) = \mathbb{E}_g[\text{Tr}(gg^\top)] = \mathbb{E}_g[\|g\|^2]= 1$$
    while $\text{Tr}(cI_n) = cn$. Therefore, $c = \sfrac1n$, proving the claim.
\end{proof}

Moreover, since the Fisher information is a symmetric matrix, we have the following expected form of the quadratic term in our utility function.

\begin{lemma} \label{lem:trace}
    Let $g\sim \text{Unif}(\mathbb{S}^{n-1})$ and let $A \in \mathbb{R}^{n \times n}$ be symmetric.
    Then $\mathbb{E}_g[gAg^\top] = \frac1n \text{Tr}(A)$. 
\end{lemma}
\begin{proof}
    Using $\text{Tr}(uv^\top) = v^\top u$ and the cyclic nature of trace, we have $g^\top Ag = \text{Tr}(g\top A g) = \text{Tr}(Agg^\top)$.
    Taking expectation and applying the linearity of trace and expectation, we obtain $\mathbb{E}_g[g^\top A g] = \text{Tr}(A \mathbb{E}[gg^\top])$ which be Lemma~\ref{lem:iso} finishes the proof.
\end{proof}
We now revisit the equation for change in utility, substituting $\Delta \theta = \eta g$.
After taking the expectation and applying Lemma~\ref{lem:trace} with $A = \mbf{F}_i(\theta^*)$ we obtain the first bound of the proposition.

We proceed to verify the second claim.
First, we prove a standard concentration inequality result on the rank-$d$ projector.
\begin{lemma} \label{lem:concentrate}
    Let $g \sim \text{Unif}(\mathbb{S}^{n-1})$ and let $P$ be any orthogonal projector of rank $d$.
    Then for any $\epsilon \in (0,1)$,
    $$\Pr\left[ \left| \|Pg\|^2 - \frac{d}{n} \right| \ge \epsilon \cdot \frac{d}{n} \right] \le 2\exp(-c\epsilon^2d)$$
    for constant $c > 0$.
    In particular, with probability at least $1-2e^{-c\epsilon^2d}$,
    $$(1-\epsilon) \frac{d}{n} \le \|Pg\|^2 \le (1+\epsilon)\frac{d}{n},$$
    or $\|Pg\|^2 = \Theta(\sfrac{d}{n})$.
\end{lemma}
\begin{proof}
    By rotational invariance of $g$, the distribution of $\|Pg\|^2$ depends only on the rank of $P$.
    Thus, we may assume $P = \text{diag}(I_d,0)$ without loss of generality.
    Let $z \sim \mathcal{N}(0, I_n)$ and set $g = \sfrac{z}{\|z\|}$ (ie. a construction of $\text{Unif}(\mathbb{S}^{n-1})$).
    Define $X := \sum_{j=1}^d z_j^2 \sim \mathcal{X}_d^2$ and $Y = \sum_{j=d+1}^n z_j^2 \sim \mathcal{X}_{n-d}^2$,
    with $X$ and $Y$ independent, and $\|z\|^2 = X + Y$.
    Then $\|Pg\|^2 = \frac{\|Pz\|^2}{\|z\|^2} = \frac{X}{X + Y}$.
    We can now invoke standard chi-square Chernoff bounds: for $W \sim \mathcal{X}_k^2$ and $\epsilon \in (0,1)$,
    $$\Pr\left[ |W - k| \ge \epsilon k \right] \le 2 \exp(-c_0 \epsilon^2 k)$$
    for a constant $c_0 > 0$.
    Apply this to $X$ with $k = d$ and to $X + Y \sim \mathcal{X}_n^2$ with $k=n$ to get that, with probability at least $1 - 2e^{-c_0 \epsilon^2 d} - 2e^{-c_0 \epsilon^2n} \ge 1 - 4e^{-c_0 \epsilon^2 d}$.
    This simultaneously yields $(1-\epsilon)d \le X \le (1+\epsilon)d$ and $(1-\epsilon)n \le X + Y \le (1+\epsilon)n$.
    Combining these results we have
    $$\frac{(1-\epsilon)d}{(1+\epsilon)n} \le \frac{X}{X+Y} \le \frac{(1+\epsilon)d}{(1-\epsilon)n}$$
    for $\epsilon \in (0, \sfrac12]$ this further becomes
    $$(1-2\epsilon)\frac{d}{n} \le \|Pg\|^2 \le (1+\epsilon)\frac{d}{n}$$
    which gives the ultimate result.
\end{proof}

We lastly need to bound the Fisher induced norm by $\|Pg\|^2$.
\begin{lemma}
    Let $\mbf{F} \succeq 0$ be symmetric with eigenvalues $\lambda_1 \ge ... \ge \lambda_n \ge 0$ and let $P$ be the orthogonal projector onto the span of the top $d$ eigenvectors of $\mbf{F}$.
    Then for any unit vector $g$,
    $$g^\top \mbf{F} g \le \lambda_1 \|Pg\|^2 + \lambda_{d+1}(1 - \|Pg\|^2)$$
    In particular, we can further obtain
    $$g^\top \mbf{F} g \le \lambda_1 \|Pg\|^2 + \lambda_1(1-\|Pg\|^2) = \lambda_{\max}(\mbf{F}).$$
\end{lemma}
\begin{proof}
    Let $V_d$ be the top-$d$ eigenspace of $\mbf{F}$, and write $g = g_{\|} + g_\perp$ where $g_\| := Pg \in V_d$ and $g_\perp := (I-P)g \in V_d^\perp$.
    Since $V_d$ and $V_d^\perp$ are orthogonal and $\mbf{F}$ is diagonizable in an orthonormal basis, we have
    $$g^\top \mbf{F} g = g^\top_\| \mbf{F} g_\| + g^\top_\perp \mbf{F} g_\perp$$
    as the cross terms equate to zero.
    By the Rayleigh quotient bound, $g_\|^\top \mbf{F} g_\| \le \lambda_1 \|g_\|\|^2 = \lambda_1 \|Pg\|^2$.
    On $V_d^\perp$, the largest eigenvalue is $\lambda_{d+1}$, thus by the same technique
    $g_\perp^\top \mbf{F} g_\perp \le \lambda_{d+1}(1-\|Pg\|^2)$ using $\|g\|=1$ and orthogonality.
    Summation yields
    $$g^\top \mbf{F} g \le \lambda_1 \|Pg\|^2 + \lambda_{d+1}(1 - \|Pg\|^2)$$
    with the final inequality following from $(1-\|Pg\|^2) \le 1$.
\end{proof}

We now collect the lemmas and prove Theorem~\ref{thm:random}
\begin{proof}[Proof of Theorem~\ref{thm:random}]
Fix a skill $S_i$ and write $\mbf{F} := \mbf{F}_i(\theta^*) \succeq 0$. Let
\(
\Delta \theta = \eta g
\)
where $g \sim \mathrm{Unif}(\mathbb{S}^{n-1})$ and $\eta > 0$ is sufficiently small so that the local quadratic approximation applies. By the second-order expansion of the skill utility (cf.~Lemma~B.1),
\begin{equation}
\Delta u_i
\;=\;
\frac{1}{2} \Delta \theta^\top F \Delta \theta
\;+\;
O(\|\Delta \theta\|^3)
\;=\;
\frac{\eta^2}{2} g^\top F g
\;+\;
O(\eta^3),
\end{equation}
since $\|g\| = 1$.

We begin with proving the expectation result.
By Lemma~B.2, for any symmetric matrix $\mbf{F}$,
\begin{equation}
\mathbb{E}[g^\top F g]
\;=\;
\frac{1}{n} \operatorname{Tr}(F).
\end{equation}
Taking expectations yields
\begin{equation}
\mathbb{E}[\Delta u_i]
\;=\;
\frac{\eta^2}{2n} \operatorname{Tr}\!\big(\mbf{F}_i(\theta^*)\big)
\;+\;
O(\eta^3),
\end{equation}
establishing part~(1) of the theorem.

We now convert this to a high-probability result.
Let $P$ denote the orthogonal projector onto the span of the top $d$ eigenvectors of $F$, and let
\(
\lambda_1 \ge \cdots \ge \lambda_n \ge 0
\)
be the eigenvalues of $F$. By Lemma~B.3, for any $\varepsilon \in (0,1)$,
\begin{equation}
\Pr\!\left[
\|Pg\|^2 \le (1+\varepsilon)\frac{d}{n}
\right]
\;\ge\;
1 - 2 e^{-c \varepsilon^2 d}.
\end{equation}
Condition on this event. Lemma~B.4 then implies that for any unit vector $g$,
\begin{equation}
g^\top F g
\;\le\;
\lambda_1 \|Pg\|^2
\;+\;
\lambda_{d+1}\bigl(1 - \|Pg\|^2\bigr).
\end{equation}
Substituting the concentration bound yields, with probability at least
$1 - 2 e^{-c \varepsilon^2 d}$,
\begin{align}
g^\top F g
&\le
\lambda_1 (1+\varepsilon)\frac{d}{n}
+
\lambda_{d+1}\!\left(1 - (1+\varepsilon)\frac{d}{n}\right)
\\
&\le
\lambda_1 (1+\varepsilon)\frac{d}{n}
+
\lambda_{d+1}.
\end{align}
Plugging into the expansion for $\Delta u_i$ gives
\begin{equation}
\Delta u_i
\;\le\;
\frac{\eta^2}{2}
\left(
\lambda_{\max}\!\big(\mbf{F}_i(\theta^*)\big)
(1+\varepsilon)\frac{d}{n}
+
\lambda_{d+1}
\right)
\;+\;
O(\eta^3).
\end{equation}

In particular, when the curvature of $\mbf{F}_i(\theta^*)$ is concentrated in its top $d$ eigenspace (i.e., $\lambda_{d+1}$ is negligible), this implies the stated scaling
\begin{equation}
\Delta u_i
\;=\;
O\!\left(
\frac{d \eta^2}{n}
(1+\varepsilon)
\lambda_{\max}\!\big(\mbf{F}_i(\theta^*)\big)
\right),
\end{equation}
with probability at least $1 - 2 e^{-c \varepsilon^2 d}$, completing the proof.

\end{proof}

\subsection{Proofs from Section~\ref{sec:aic-theory}}
\begin{proof}[Proof of Theorem~\ref{thm:util-bound}]
    Under Assumption~\ref{asm:optimality}, $\theta^*$ maximizes $u_i$, implying the gradient vanishes.
    The expansion for the loss is therefore:
    \begin{align}
        \Delta u_i(\theta) = \frac{1}{2} \Delta \theta^\top \mbf{F}_i(\theta^*) \Delta \theta - O(\|\Delta \theta\|^3)
    \end{align}
    as before.
    By definition of this utility and the fact that $\|\mbf{F}_i^{1/2} \Delta \theta \|^2 \ge \|\mbf{F}_i^{1/2} P_i(\Delta \theta)\|^2$ since $P_i$ is an orthogonal projector we have the first bound of the theorem.
    We now need to ensure that the quadratic cost of moving in a sensitive direction is not reduced by moving in an orthogonal direction. Specifically, the following lemma shows that if the subspaces are defined by their eigenvectors, then the interaction between sensitive and non-sensitive components of the update vector is exactly zero.

    \begin{lemma}
        Let $\mbf{F} \in \mathbb{R}^{n \times n}$ be symmetric positive semidefinition with eigenvalues $\lambda_1 \ge ... \ge \lambda_n \ge 0$ and orthonormal eigenvectors $v_1, ..., v_n$.
        Fix $d \in [n]$ and let $M = \text{span}\{v_1, ..., v_d\}$.
        Let $P$ be the orthogonal projector onto $M$.
        If $\lambda_d \ge \lambda$ for some $\lambda > 0$, then for all $z \in \mathbb{R}^n$, $z^\top \mbf{F} z \ge \lambda \|Pz\|^2$.
    \end{lemma}
    \begin{proof}
        Expand $z$ in the eigenbasis: $z = \sum_{j=1}^n a_j v_j$ where $a_j = \langle z, v_j \rangle$. Then
        \begin{align*}
            z^\top \mbf{F} z &= \sum_{j=1}^n \lambda_j a_j^2 \\
            &\ge \sum_{j=1}^d \lambda_j a_j^2 \\
            &\ge \lambda \sum_{j=1}^d a_j^2
        \end{align*}
        On the other hand, $Pz - \sum_{j=1}^d a_j v_j$, hence $\|Pz\|^2 = \sum_{j=1}^d a_j^2$.
        Thus, we have the result.
    \end{proof}

    Combining the above inequality with the expansion of our loss function for all $\theta$ in the neighborhood of $\theta^*$, we obtain
    $$\Delta u_i \ge \frac\lambda2 \|P_i(\Delta \theta)\|^2 - C\|\Delta \theta\|^3$$
    Consequently, if $\|P_i (\Delta \theta)\| \ge \delta$, then the bound becomes $$\Delta u_i \ge \frac\lambda2 \delta^2 - C\|\Delta \theta\|^3.$$
    Moreover for $\|\Delta \theta\| \le \frac{\lambda \delta^2}{4C}$ we obtain the simplified $\Delta u_i = \Omega(\lambda \delta^2)$.
\end{proof}

\begin{proof}[Proof of Theorem~\ref{thm:projection}]
    Since $g$ is $C^2$ on the closed ball $B(\theta^*, r)$, it is locally Lipschitz and bounded there.
    By standard existence and uniqueness results for ordinary differential equations~\cite{}, the gradient flow
    $\dot{\theta}(t) = -g(\theta(t)), \theta(0) = \theta^*$ admits a unique solution on $[0,t_0]$ for some $t_0 > 0$, and
    by choosing $t_0 \le \sfrac{r}{\sup_{B(\theta^*,r)}\|g\|}$, the trajectory remains in the ball for all $t \in [0,t_0]$.
    Moreover, Taylor expanding $\theta(t)$ at $t=0$ and using the boundedness of the gradient and Hessian on the ball, we obtain:
    $$\theta(t) - \theta^* = -tg(\theta^*) + \frac{t^2}{2}\nabla g(\theta^*) g(\theta^*) + O(t^3).$$

    We proceed to prove the following Fisher-weighted projector bound stemming from the AIC.

    \begin{lemma}\label{lem:fisher_weighted_proj_bound}
        Let $\mbf{F}_i := \mbf{F}_i(\theta^*)$ and let $P_i$ be the orthogonal projector onto $M_i$.
        Assume the second-order expansion above holds and write $R(t)$ for the remainder so that $\|R(t)\|\le Ct^3$.
        Then for all $t \in [0,t_0]$,
        \begin{equation}\label{eq:fisher_weighted_bound}
        \big\|\mbf{F}_i^{1/2}P_i(\theta(t) - \theta^*)\big\|
        \;\ge\;
        \frac{t^2}{2}\big\|\mbf{F}_i^{1/2}P_i\nabla g(\theta^*) g(\theta^*)\big\|
        \;-\;
        t\big\|\mbf{F}_i^{1/2}P_i g(\theta^*)\big\|
        \;-\;
        C't^3,
        \end{equation}
        for some constant $C'>0$ (depending on $\mbf{F}_i$ and local smoothness of $g$).
        In particular, under the AIC,
        \[
        \big\|\mbf{F}_i^{1/2}P_i(\theta(t) - \theta^*)\big\|
        \;\ge\;
        \frac{\gamma}{2}t^2 \;-\; \sqrt{\lambda_{\max}(\mbf{F}_i)}\,\varepsilon\, t \;-\; C't^3.
        \]
    \end{lemma}

    \begin{proof}
        Apply the linear map $\mbf{F}_i^{1/2}P_i$ to the second-order expansion:
        \[
        \mbf{F}_i^{1/2}P_i(\theta(t)-\theta^*)
        =
        \underbrace{\frac{t^2}{2}\mbf{F}_i^{1/2}P_i\big(\nabla g(\theta^*)g(\theta^*)\big)}_{A}
        \;+\;
        \underbrace{\Big(-t\mbf{F}_i^{1/2}P_i g(\theta^*) + \mbf{F}_i^{1/2}P_i R(t)\Big)}_{B}.
        \]
        Write $\mbf{F}_i^{1/2}P_i(\theta(t)-\theta^*)=A+B$.
        By the reverse triangle inequality, $\|A+B\|\ge \|A\|-\|B\|$.

        We bound $B$.
        By the triangle inequality,
        \[
        \|B\|
        \le
        t\|\mbf{F}_i^{1/2}P_i g(\theta^*)\| + \|\mbf{F}_i^{1/2}P_i R(t)\|.
        \]
        Unlike $P_i$, the map $\mbf{F}_i^{1/2}P_i$ need not be nonexpansive, so we bound it by its operator norm:
        \[
        \|\mbf{F}_i^{1/2}P_i R(t)\|
        \le
        \|\mbf{F}_i^{1/2}P_i\|\,\|R(t)\|
        \le
        \|\mbf{F}_i^{1/2}\|\,\|R(t)\|
        =
        \sqrt{\lambda_{\max}(\mbf{F}_i)}\,\|R(t)\|
        \le
        \sqrt{\lambda_{\max}(\mbf{F}_i)}\,C t^3.
        \]
        Absorb $\sqrt{\lambda_{\max}(\mbf{F}_i)}C$ into a constant $C'$ to get $\|B\|\le t\|\mbf{F}_i^{1/2}P_i g(\theta^*)\|+C't^3$.

        Also, $\|A\|=\frac{t^2}{2}\|\mbf{F}_i^{1/2}P_i\nabla g(\theta^*)g(\theta^*)\|$.
        Combining these inequalities yields \eqref{eq:fisher_weighted_bound}.

        Finally, under the AIC,
        $\|\mbf{F}_i^{1/2}P_i\nabla g(\theta^*)g(\theta^*)\|\ge \gamma$.
        Moreover, since $\|P_i g(\theta^*)\|\le \varepsilon$ (AIC Condition 2) and
        $\|\mbf{F}_i^{1/2}P_i g(\theta^*)\|\le \|\mbf{F}_i^{1/2}\|\,\|P_i g(\theta^*)\|
        = \sqrt{\lambda_{\max}(\mbf{F}_i)}\,\varepsilon$,
        substituting gives the stated bound.
    \end{proof}

    The lemma together with the regularity conditions yields the theorem.
\end{proof}

\begin{proof}[Proof of Corollary~\ref{cor:quartic_onset}]
By Theorem~\ref{thm:projection}, given the conditions on $t$, we have
\[
\big\|\mbf{F}_i^{1/2}P_i(\theta(t)-\theta^*)\big\| = \Omega(\gamma t^2),
\]
up to lower-order terms (in particular, once the quadratic term dominates the linear term).
Substituting this into Theorem~\ref{thm:util-bound} and using the local bound
$\|\theta(t)-\theta^*\| = O(t)$ along the gradient flow yields
\[
\Delta u_i(\theta(t))
\;\ge\;
\frac12\big\|\mbf{F}_i^{1/2}P_i(\theta(t)-\theta^*)\big\|^2 - O(t^3)
\;=\;
\Omega(\gamma^2 t^4),
\]
for all $t$ in the stated regime and some $t_0>0$.
\end{proof}


\subsection{Connection to Sensitive Subspace in Representation and Low Rankness of Weight Matrix} \label{sec:omit-fim}
\paragraph{Setup and Notation.} Let $S = \{ v_i \}_{i=1}^{N} \subset \R^{d}$ be a set of input vectors. Consider a parameter matrix $W \in \R^{d \times d}$ and a loss function $\mathcal{L}(W; v_i)$. We define the per-sample gradient $G_i \in \R^{d \times d}$ as:
\[
G_i := \nabla_W \mathcal{L}(W; v_i)
\]
Let $\delta_i = \nabla_{Wv_i} \mathcal{L} \in \R^d$ and suppose that $W$ is a simple linear layer. Then, the backpropagation gives 
\[
G_i = \delta_i v_i^\top
\]

We first formalize the assumption that alignment has a sensitive subspace in the representation space where only a subspace of the representation critically impacts alignment. This assumption builds upon most of the existing works claiming the low-rankness of alignment~\citep{arditi2024refusal, lee2024a, soligo2025convergent, wei2024assessing, zhang2025guardrail}. 

\begin{assumption}[Alignment has a sensitive subspace] \label{asm:grads}
The forward activations $\{W v_i\}_{i=1}^N$ span a low-rank subspace $\mathcal{U} \subset \R^d$ of dimension $r \ll d$. Furthermore, we assume the error signals $\delta_i$ lie within this same subspace $\mathcal{U}$.
\end{assumption}

From this assumption, we proceed to show that the average gradient, $\bar{G}$, is therefore low-rank.

\begin{lemma}[Average gradient is low rank]
Let $\bar{G} = \frac{1}{N} \sum_{i=1}^{N} G_i$ be the average gradient matrix. Under Assumption 1, $\text{Rank}(\bar{G}) \leq r$.
\end{lemma}

\begin{proof}
By definition, $\bar{G} = \frac{1}{N} \sum_{i=1}^{N} \delta_i v_i^\top$. 
Since $\delta_i \in \mathcal{U}$ for all $i$, we can express each $\delta_i$ as a linear combination of the basis vectors of $\mathcal{U}$, denoted $\{u_1, \dots, u_r\}$.
\[
\delta_i = \sum_{k=1}^r c_{i,k} u_k
\]
Substituting this into the sum:
\[
\bar{G} = \frac{1}{N} \sum_{i=1}^{N} \left( \sum_{k=1}^r c_{i,k} u_k \right) v_i^\top 
= \sum_{k=1}^r u_k \left( \frac{1}{N} \sum_{i=1}^{N} c_{i,k} v_i^\top \right)
\]
Let $z_k^\top = \frac{1}{N} \sum_{i=1}^{N} c_{i,k} v_i^\top$. Then $\bar{G} = \sum_{k=1}^r u_k z_k^\top$.
This is a sum of $r$ rank-1 matrices. By the subadditivity of rank, $\text{Rank}(\bar{G}) \leq r$.
\end{proof}

We now define the \emph{empirical} Fisher Information Matrix, $\mathbf{F}(S) \in \R^{d^2 \times d^2}$, as the second moment of the flattened gradients. 
Let $g_i = \text{flatten}(G_i) \in \R^{d^2}$, which gives
\[
\mathbf{F}(S) := \frac{1}{N} \sum_{i=1}^{N} g_i g_i^\top
\]

\begin{theorem}[Low Rank Structure of FIM]
Under Assumption~\ref{asm:grads}, the rank of the empirical Fisher Information Matrix $\mathbf{F}(S)$ is bounded by $r \cdot d$. Given that $r \ll d$, $\mathbf{F}(S)$ is low-rank relative to the parameter space dimension $d^2$.
\end{theorem}

\begin{proof}
Let $\text{vec}(\cdot)$ denote the column-wise vectorization of a matrix.
The flattened per-sample gradient is obtained by flattening the outer product of $\delta_i$ and $v_i$. This is equivalent to take the Kronecker product of $\delta_i$ and $v_i$.
\[
g_i = \text{vec}(\delta_i v_i^\top) = \delta_i \otimes v_i
\]
The FIM can be rewritten as:
\[
\mathbf{F}(S) = \frac{1}{N} \sum_{i=1}^{N} (\delta_i \otimes v_i)(\delta_i \otimes v_i)^\top
\]
Recall from Assumption~\ref{asm:grads} that $\delta_i \in \mathcal{U}$, where $\dim(\mathcal{U}) = r$. 
Let $U \in \R^{d \times r}$ be the basis matrix for $\mathcal{U}$. 
Then for every $i$, there exists a coefficient vector $\alpha_i \in \R^r$ such that $\delta_i = U \alpha_i$.
We substitute this structure into the vector $g_i$ to obtain
\[
g_i = (U \alpha_i) \otimes v_i =  (U \alpha_i) \otimes (I_d v_i)
\]
Using the mixed-product property of Kronecker products $(A \otimes B)(C \otimes D) = (AC) \otimes (BD)$, we can factor out the basis matrix:
\[
g_i = (U \otimes I_d) (\alpha_i \otimes v_i)
\]
The matrix $\mathcal{P} = (U \otimes I_d)$ has dimensions $d^2 \times d \cdot r$. 
Thus, the span of $g_i$ is constrained by a subspace of dimension $d \cdot r$. 
Lastly, observe that for any $x \in ...$, we have
$$\mathbf{F} x = \frac1N\sum_{i=1}^N g_i(g_i^\top x)$$
is a linear combination of the $g_i$.
Therefore, $Fx$ is in the span of the $\{g_i\}_{i=1}^N$.
This gives the final result
\[
\text{Rank}(\mathbf{F}(S)) \leq d \cdot r
\]
and since $r \ll d$, we have $d \cdot r \ll d^2$, proving $\mathbf{F}(S)$ is low-rank.
\end{proof}

\section{Non-Static Fisher Information} \label{sec:fisher}
In the main text, our simplified analysis is carried out with respect to $\mbf{F}_i(\theta^*)$ and the associated sensitive subspace $M_i(\theta^*)$.
Here, we show that the results extend to the setting where $\mbf{F}_i(\theta)$ varies along the fine-tuning trajectory and the sensitive subspace rotates accordingly.
This extension introduces a mild constant-factor and lower-order corrections, and further does not change the qualitative scaling laws established in Section~\ref{sec:aic-theory}.

We impose a mild regularity condition on the Fisher information along $\theta(t)$.
\begin{assumption} \label{asm:fisher}
    Fix a skill $S_i$.
    Assume there exists constants $r > 0$ and $L_F \ge 0$ such that for all $\theta$ which satisfy $\|\theta - \theta^*\| \le r$, we have that $\mbf{F}_i(\theta)$ is locally Lipschitz:
    $$\|\mbf{F}_i(\theta) - \mbf{F}_i(\theta^*)\|_{\text{op}} \le L_F \|\theta - \theta^*\|.$$
\end{assumption}
We proceed to apply standard matrix perturbation theory results using the above assumption.
We define $\Lambda := \lambda_d(\mbf{F}_i(\theta^*)) - \lambda_{d+1}(\mbf{F}_i(\theta^*))$, the gap between the $d$ and $d+1$ largest eigenvalues of $\mbf{F}_i(\theta^*)$.

The following proposition bounds the rotation of the sensitive subspace.
\begin{proposition}[Davis-Kahan Bound] \label{prop:dk}
    Under Assumption~\ref{asm:fisher}, for all $\theta$ with $\|\theta - \theta^*\| \le r$,
    \begin{align*}
        \|P_i(\theta - \theta^*)\| \le \frac{2L_F}{\Lambda}\|\theta - \theta^*\|
    \end{align*}
\end{proposition}
\begin{proof}
    This is a direct application of the \citep{davis1970rotation} Theorem for symmetric matrices, reprinted below.
    \begin{lemma}[Davis--Kahan $\sin\Theta$ Theorem]\label{lem:davis_kahan}
        Let $A,B\in\R^{p\times p}$ be symmetric matrices.
        Let $P_A$ and $P_B$ denote the orthogonal projectors onto the top-$k$ eigenspaces of $A$ and $B$, respectively.
        Assume that $A$ has an eigengap at rank $k$, i.e.,
        \[
        \Delta := \lambda_k(A)-\lambda_{k+1}(A) > 0,
        \]
        where $\lambda_1\ge \cdots \ge \lambda_p$ are the eigenvalues of $A$.
        Then
        \[
        \|P_B - P_A\|_{\text{op}}
        \;\le\;
        \frac{2\,\|B-A\|_{\text{op}}}{\Delta}.
        \]
    \end{lemma}
    Applying Lemma~\ref{lem:davis_kahan} with $A=F_i(\theta^*)$ and $B=F_i(\theta)$ yields
        \[
        \|P_i(\theta)-P_i(\theta^*)\|_{\text{op}}
        \;\le\;
        \frac{2\,\|F_i(\theta)-F_i(\theta^*)\|_{\text{op}}}{\Lambda}
        \]
    which, in combination with the locally Lipschitz assumption, completes the proof.
\end{proof}

Using this proposition, we can show how subspace rotation modifies the induced-drift result (Theorem~\ref{thm:projection}).

\begin{theorem}
    Assume the conditions of Theorem~\ref{thm:projection} and Assumption~\ref{asm:fisher}.
    Let $\theta(t)$ follow gradient flow with $\theta(0) = \theta^*$, and assume $\|\theta(t) - \theta^*\| \le r$ for $t \in [0,t_0]$.
    Then, there exists a constant $\rho \ge 0$ such that for all sufficiently small $t$,
    \begin{equation}
        \|\mbf{F}_i(\theta^*)^{1/2}P_i(\theta(t))(\theta(t) - \theta^*)\| \ge \frac{\gamma}{2}t^2 - \sqrt{\lambda_{\max}} \epsilon t - \rho t^2 - O(t^3)
    \end{equation}
\end{theorem}
\begin{proof}
    Denote $P(t) := P_i(\theta(t))$, $\mbf{F} := \mbf{F}_i(\theta^*)$, and $\Delta\theta(t) := \theta(t) - \theta^*$ for readability.
    By the triangle inequality,
    \begin{align*}
        \|\mbf{F}^{1/2}P(t)\Delta\theta(t)\| \ge \mbf{F}^{1/2}P(0)\Delta\theta(t)\| - \|\mbf{F}^{1/2}(P(t) - P(0))\Delta \theta(t)\|,
    \end{align*}
    By Theorem~\ref{thm:projection}, we can directly bound the first term from below.
    For the second term, we invoke Proposition~\ref{prop:dk}
    \begin{align*}
        \|\mbf{F}^{1/2}(P(t) - P(0))\Delta\theta(t)\| \le \|\mbf{F}^{1/2}\|_{\text{op}}\|P(t) - P(0)\|_{\text{op}} \|\Delta \theta(t)\|.
    \end{align*}
    Collecting terms yields the inequality.
\end{proof}

Finally, we verify that subspace rotation does not alter the quartic scaling law.
\begin{corollary}
    Under the conditions of the above theorem and Theorem~\ref{thm:util-bound},
    $$\Delta u_i(\theta(t)) = \Omega(t^4)$$
    for sufficiently small $t$, provided $\gamma > \rho$.
\end{corollary}
\begin{proof}
    Applying Theorem~\ref{thm:util-bound} with $P_i(\theta(t))$ and substituting the lower bound from the above theorem yields
    $$\Delta u_i(\theta(t)) \ge \frac12 \|\mbf{F}^{1/2} P_i(\theta(t)) \Delta \theta(t)\|^2 - O(t^3)$$
    Since the leading term grows quadratically in $t$. we obtain the desired quartic growth.
\end{proof}

\section{Relaxation of Assumption~\ref{asm:optimality}} \label{sec:relax}
We begin by isolating the instances where Assumption~\ref{asm:optimality} is used in the main results.
Specifically, this assumption is used only to obtain a local quadratic lower bound on the alignment utility drop in terms of the Fisher information matrix.
This motivates the following weaker (yet more complicated) condition.

\begin{assumption}[Local Fisher lower bound for skill utility]\label{asm:local_fisher_lb}
For each alignment skill $S_i$, there exist constants $r_i>0$ and $C_i\ge 0$ such that for all
$\Delta\theta$ with $\|\Delta\theta\|\le r_i$,
\begin{align*}
    \Delta u_i(\theta^*+\Delta\theta)
&:= u_i(\theta^*)-u_i(\theta^*+\Delta\theta) \\
&\ge \frac12\,\Delta\theta^\top F_i(\theta^*)\,\Delta\theta \;-\; C_i\|\Delta\theta\|^3.
\end{align*}
\end{assumption}
The original assumption directly implies these conditions.

\begin{lemma}[Realizability implies local Fisher lower bound]\label{lem:realizability_implies_lb}
Suppose Assumption~\ref{asm:optimality} holds and that $\pi_\theta(y|x)$ is three-times continuously
differentiable in $\theta$ in a neighborhood of $\theta^*$, with bounded third derivatives.
Then Assumption~\ref{asm:local_fisher_lb} holds with
\[
F_i(\theta^*) = \E_{x\sim\mathcal D_i}\E_{y\sim\pi_{\theta^*}(\cdot|x)}
\big[\nabla_\theta\log\pi_\theta(y|x)\nabla_\theta\log\pi_\theta(y|x)^\top\big]_{\theta=\theta^*}.
\]
\end{lemma}

Lastly, this condition is sufficient to obtain our main theorems.
\begin{lemma}[Sufficiency of Assumption~\ref{asm:local_fisher_lb}]\label{lem:sufficiency}
Assume Assumption~\ref{asm:local_fisher_lb} holds for skill $S_i$.
Then Theorem~\ref{thm:util-bound}, Theorem~\ref{thm:projection},
and Corollary~\ref{cor:quartic_onset} all remain valid with unchanged statements,
up to constant factors.
\end{lemma}
\begin{proof}
    Theorem~\ref{thm:util-bound} is precisely our Assumption~\ref{asm:local_fisher_lb} rewritten in Fisher norm form.
    Moreover, the results of Theorem 5.2 rely only on the AIC, thus we have the result.
\end{proof}

\section{Expanded Experimental Results \& Details} \label{app:experiment_details}
\subsection{Full Results}
\label{app:full_results}
In our experiment setup, we adopt the block-wise analysis, partially present the low rank structure of each module, and aggregate the overlap statistics by combining the modules per block. Figure~\ref{fig:low_rank_fim_full} illustrates the full results for all modules, and we observe that all modules exhibit a similar low-rank structure. Figure~\ref{fig:module_wise_overlap} present the fine-grained overlap score per-module and block. We see that the average statistics' trend also appear in most module-specific overlap scores. Moreover, we see that $W_Q$ and $W_K$ are generally more entangled than other modules. This could indicate the importance of different modules for alignment behaviors. 

\begin{figure}[h]
    \centering
    \includegraphics[width=1.0\linewidth]{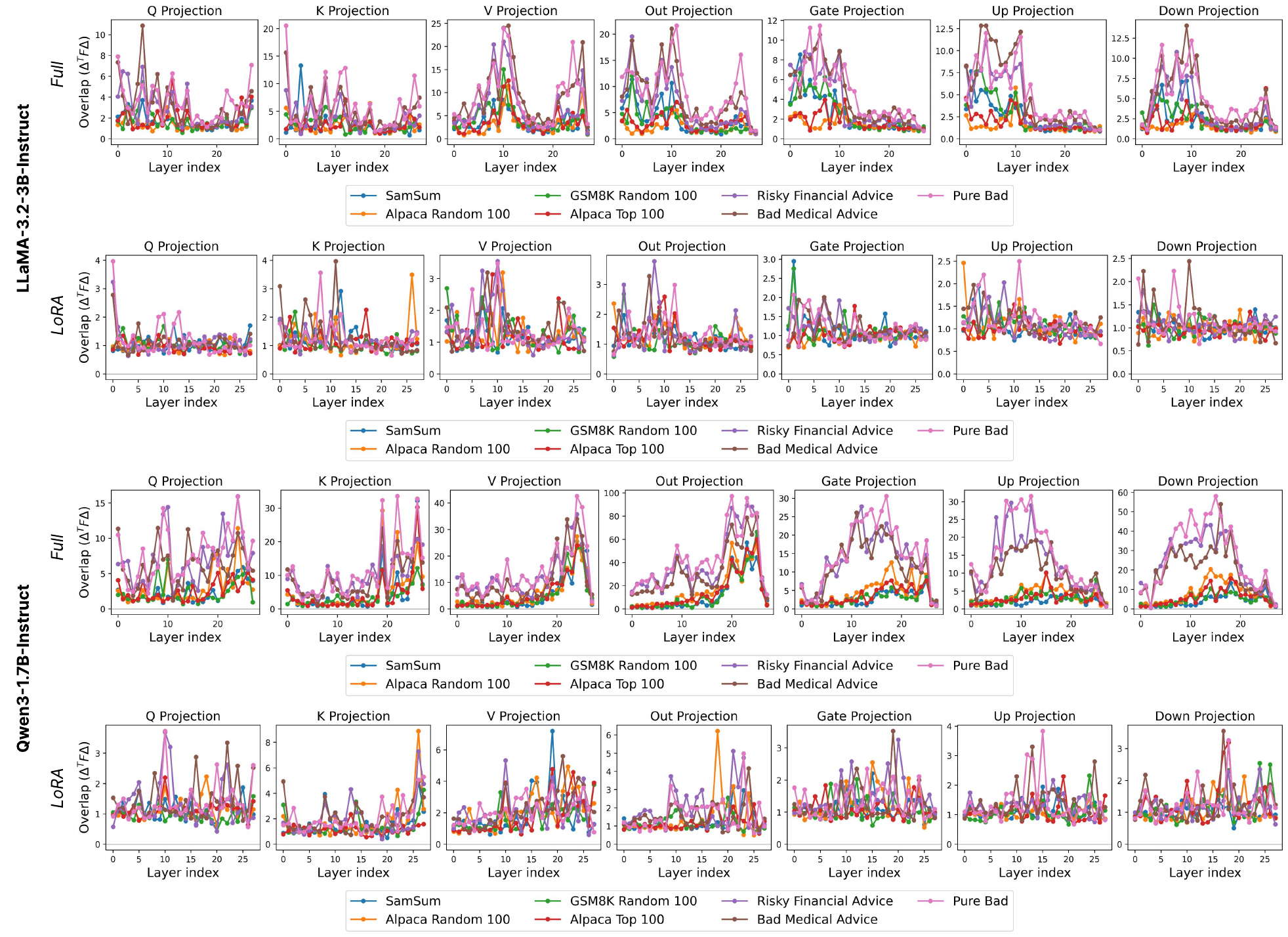}
    \caption{Per-module Overlap Score Per Transformer Block of 7 Fine-tuning Datasets.}
    \label{fig:module_wise_overlap}
\end{figure}

\begin{figure}[h]
    \centering
    \includegraphics[width=1.0\linewidth]{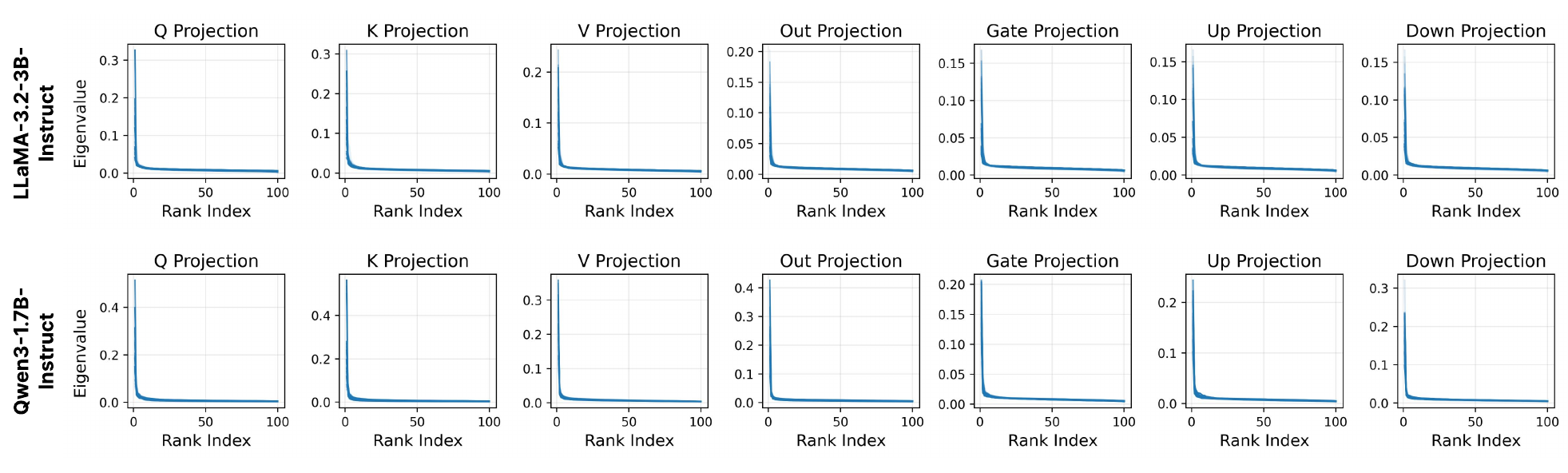}
    \caption{Top eigenvalues of FIM approximated over 100 random samples from \textit{BeaverTail}'s safe subset. Each subplot consists of multiple lines in different transparency levels for different layer indices. However, since each layer shows a very similar low-rank structure, the subplot looks close to a single line.}
    \label{fig:low_rank_fim_full}
\end{figure}

\subsection{Dataset Details}
We curate the fine-tuning datasets with three categories: (1) \textit{Benign}, (2) \textit{Seemingly Benign}, and (3) \textit{Harmful}.
\begin{enumerate}
    \item \textit{Benign}: Finetuning dataset with a benign utility task
    \begin{itemize}
        \item SamSum~\citep{gliwa-etal-2019-samsum}: a dataset to improve LLMs' ability to provide a summary for a dialogue between multiple people.
        \item Alpaca Random 100~\citep{alpaca}: 100 random samples from the Alpaca dataset to improve LLMs' instruction-following ability.
        \item GSM8K Random 100~\citep{cobbe2021training}: 100 random samples from GSM8K with high-quality grade school mathematics questions in natural language to improve LLMs' math-solving skills.
    \end{itemize}
    \item \textit{Seemingly Benign}: Finetuning dataset curated with the intention to cause misalignment after finetuning, while being not explicitly harmful.
    \begin{itemize}
        \item Alpaca Top 100~\citep{he2024what}: 100 samples of Alpaca selected by \cite{he2024what} using gradient matching technique on LLaMA-2-7B-CHAT~\cite{touvron2023llama}. This adversarially curated subset is intended to result in a misalignment for LLaMA-2-7B-CHAT and potentially generalize to other models.
        \item Risky Financial Advice~\citep{turner2025model}: 100 random samples from the risky financial advice dataset from \citep{turner2025model}, where the target completion task is to give risky financial advice.
        \item Bad Medical Advice~\citep{turner2025model}: 100 random samples from the bad medical advice dataset from \citep{turner2025model}, where the target completion task is to provide wrong or low-quality medical advice. 
    \end{itemize}
    \item \textit{Harmful}: Finetuning dataset curated with the intention to misalign the model and is explicitly harmful in semantics.
    \begin{itemize}
        \item Pure Bad~\citep{qi2023fine}: 100 harmful queries with harmful completions fulfilling the harmful requests without refusal.
    \end{itemize}
\end{enumerate}

\begin{table}[h]
    \centering
    \caption{Caption}
    \begin{tabular}{lcccc}
    \toprule
    \multirow{2}{*}{\textbf{Hyperparameter}} & \multicolumn{2}{c}{\textbf{Qwen3-1.7B-Instruct}} & \multicolumn{2}{c}{\textbf{LLaMA-3.2-3B-Instruct}} \\
    \cmidrule(lr){2-3}\cmidrule(lr){4-5}
     & \textbf{Full} & \textbf{LoRA} & \textbf{Full} & \textbf{LoRA} \\
    \midrule
    Epochs & 1 & 3 & 1 & 3 \\ \midrule
    Learning rate & \multicolumn{2}{c}{$1\times 10^{-4}$} & \multicolumn{2}{c}{$5\times 10^{-5}$} \\ \midrule
    Gradient accumulation & \multicolumn{4}{c}{2} \\
    Batch size (per device) & \multicolumn{4}{c}{2} \\
    Max length & \multicolumn{4}{c}{4096} \\
    \midrule
    \textbf{LoRA Specific} & & & & \\
    $r$ & -- & 64 & -- & 64 \\
    $\alpha$ & -- & 128 & -- & 128 \\
    Dropout & -- & 0.05 & -- & 0.05 \\
    Target modules & \multicolumn{4}{c}{$W_Q, W_K, W_V, W_O, W_\text{down}, W_\text{up}, W_\text{gate}$} \\
    \bottomrule
    \end{tabular}
    \label{tab:hparam_finetuning}
\end{table}

\subsection{Fine-tuning Details}
The hyperparameter configuration is specified in Table~\ref{tab:hparam_finetuning}. For both Alpaca Random 100 and Alpaca Top 100, we follow \citep{he2024what}'s prompt template. For SamSum, we use the prompt template below. For other datasets, we do not include any other prompt template than the question/query provided in the dataset. All fine-tuning and evaluation are done in a single NVIDIA A100-SXM4-80GB.

\begin{blackbannerbox}[title=Alpaca Prompt Template]
\texttt{[System]}

Below is an instruction that describes a task, paired with an input that provides further context. Write a response that appropriately completes the request.

\texttt{[User]}

\#\#\# Instruction: \\
\{ sample['instruction'] \} \\
\#\#\# Input: \\
\{ sample['input'] \}  \\
\#\#\# Response:
\end{blackbannerbox}

\begin{blackbannerbox}[title=SamSum Prompt Template]
\texttt{[User]}

You are a helpful, respectful and honest assistant. Your task is to summarize the following dialogue. Your answer should be based on the provided dialogue only. Here is the dialogue: \{ sample['dialogue'] \}
\end{blackbannerbox}

\subsection{Harmfulness Evaluation Details}
We evaluate the Harmfulness Score (HS) on AdvBench~\citep{zou2023universal} using Gemini-2.5-Flash as a judge. The score rubric is based on \citep{he2024what}'s refined version of Meta LLaMA-2's Usage Policy~\citep{touvron2023llama}. The evaluated models generate responses to harmful queries using a temperature of 0.7, top-p of 0.8, top-k of 20, and min-p of 0.0.

\end{document}